\documentclass{article}

\usepackage{microtype}
\usepackage{graphicx}
\usepackage{subcaption}
\usepackage{booktabs} 

\usepackage{hyperref}

\usepackage[preprint]{icml2026}

\usepackage{amsmath}
\usepackage{amssymb}
\usepackage{mathtools}
\usepackage{amsthm}

\usepackage{wrapfig}
\usepackage{caption}
\usepackage{makecell}
\usepackage{pifont}

\usepackage{stmaryrd}
\renewcommand{\mathcal}{\mathscr}
\usepackage[mathscr]{eucal}
\newcommand{\mbf}{\mathbf}
\newcommand{\inv}{^{-1}}
\newcommand{\diag}{\operatorname{diag}}
\newcommand{\symP}{\tilde{\mbf P}}
\newcommand{\tentry}[3]{$#1_{-#2}^{+#3}$}
\newcommand{\btentry}[3]{$\mbf{#1}_{-\mbf{#2}}^{+\mbf{#3}}$}
\newcommand{\E}{\mathbb{E}}
\newcommand{\placeholder}{\mathrel{\phantom{=}}}
\newcommand{\tick}{\ding{51}}
\newcommand{\cross}{\ding{55}}

\usepackage[capitalize,noabbrev]{cleveref}

\theoremstyle{plain}
\newtheorem{theorem}{Theorem}[section]
\newtheorem{proposition}[theorem]{Proposition}

\newtheorem{corollary}[theorem]{Corollary}
\theoremstyle{definition}

\theoremstyle{remark}
\newtheorem{remark}[theorem]{Remark}

\usepackage[textsize=tiny]{todonotes}

\icmltitlerunning{DROGO: Default Representation Objective via Graph Optimization}

\begin{document}

\twocolumn[
  \icmltitle{DROGO: Default Representation Objective via Graph Optimization\\in Reinforcement Learning}



  \icmlsetsymbol{equal}{*}

  \begin{icmlauthorlist}
    \icmlauthor{Hon Tik Tse}{uofa,amii}
    \icmlauthor{Marlos C. Machado}{uofa,amii,cifar}
  \end{icmlauthorlist}

  \icmlaffiliation{uofa}{Department of Computing Science, University of Alberta}
  \icmlaffiliation{amii}{Alberta Machine Intelligence Institute (Amii)}
  \icmlaffiliation{cifar}{Canada CIFAR AI Chair}

  \icmlcorrespondingauthor{Hon Tik Tse}{hontik@ualberta.ca}

  \icmlkeywords{Reinforcement Learning, Representation Learning, Default Representation, Successor Representation}

  \vskip 0.3in
]



\printAffiliationsAndNotice{}  

\begin{abstract}
  In computational reinforcement learning, the default representation (DR) and its principal eigenvector have been shown to be effective for a wide variety of applications, including reward shaping, count-based exploration, option discovery, and transfer. 
  However, in prior investigations, the eigenvectors of the DR were computed by first approximating the DR matrix, and then performing an eigendecomposition. This procedure is computationally expensive and does not scale to high-dimensional spaces. In this paper, we derive an objective for directly approximating the principal eigenvector of the DR with a neural network. We empirically demonstrate the effectiveness of the objective in a number of environments, and apply the learned eigenvectors for reward shaping.
\end{abstract}

\section{Introduction}
Learning representations is key to reinforcement learning (RL). The {\it successor representation}~\citep[SR;][]{dayan1993improving} is a promising temporal predictive representation, which represents a state as the expected discounted number of visits to successor states. In recent years, it has been shown to be effective in applications such as reward shaping~\citep{wulaplacian, wang2021towards} and exploration~\citep{machado2020count, machado2023temporal}. Beyond computational RL, evidence also suggests its relevance in human decision-making and neuroscience~\citep{momennejad2017successor, stachenfeld2014design, stachenfeld2017hippocampus, gomez2025representations}. \looseness=-1

The SR essentially isolates the temporal aspect of the problem from the reward. Given this property, once approximated, it can be used to perform flexible transfer to different reward functions. This idea underlies many recent transfer methods~\citep{barreto2017successor, touati2021learning}. However, for applications focusing on a single reward function, the SR can be limiting. \citet{tse2025rewardaware} have shown that when applied to such applications, the SR's reward-agnostic nature can lead to undesirable behavior and poor performance. \looseness=-1

\begin{figure}[t]
  \vskip 0.2in
  \begin{center}
    \centerline{\includegraphics[width=0.7\columnwidth]{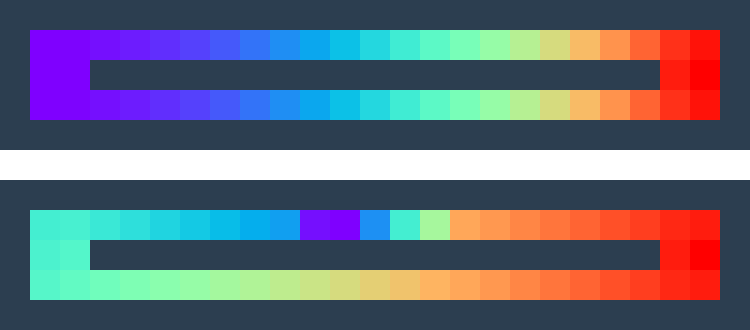}}
    \caption{
      The principal eigenvectors of the SR (top) and DR (bottom) in a ring-shaped environment with a low-region region at the top of the ring. The DR, unlike the SR, is reward-aware and captures the low-reward region in its principal eigenvector.
    }
    \label{fig:dr_vs_sr}
  \end{center}
\end{figure}

One remedy is to use reward-aware representations for these applications. In particular, the default representation \citep[DR;][]{piray2021linear} can be viewed as a reward-aware generalization of the SR (see Fig.~\ref{fig:dr_vs_sr}). While the SR captures the number of transitions required to travel between states, the DR captures the rewards obtained when traveling between states. The DR has been shown to give rise to reward-aware behavior in various applications. For example, options induced by the DR take detours to avoid low-reward regions in the environment~\citep{tse2025rewardaware}. \looseness=-1

However, the study of the DR is limited to the tabular setting. While applications such as reward shaping and option discovery rely on the principal eigenvector of the DR, it is always computed by first approximating the full DR matrix, before performing an eigendecomposition. Such a procedure is computationally expensive, and does not scale to environments with high-dimensional observations. \looseness=-1


In this paper, we address these issues by deriving a novel objective, \textbf{D}efault \textbf{R}epresentation \textbf{O}bjective via \textbf{G}raph \textbf{O}ptimization (DROGO). Using DROGO, we can train a neural network to directly approximate the principal eigenvector of the default representation. DROGO optimizes a variant of the graph-drawing objective~\citep[GDO;][]{koren2003spectral}, where we reparameterize GDO in the log-space. We derive the natural gradient~\citep{amari1998natural, amari1998natural2} to account for the inherent geometry of the log-space, which, left untreated, leads to significant instabilities during optimization. We also apply an alternative constraint on the solution instead of the usual quadratic norm penalty~\citep{wulaplacian}, and provide theoretical justification.

Further, we empirically demonstrate the effectiveness of DROGO in a set of grid world environments. We apply DROGO to learn the principal eigenvector when using different state representations as input to the neural network, such as coordinates and pixels, demonstrating DROGO's robustness. Finally, we demonstrate the utility of the learned eigenvectors by applying them for reward shaping.

\section{Preliminaries}
In this paper, we use lowercase symbols (e.g., $r, p$) to denote functions or constants (e.g., $c, \epsilon$), calligraphic font (e.g, $\mathcal{S}, \mathcal{A}$) to denote sets, bold lowercase symbols (e.g., $\mbf r, \mbf e$) to denote vectors, and bold uppercase symbols (e.g., $\mbf \Psi, \mbf Z$) to denote matrices. We index the $(i, j)$-th entry of a matrix $\mbf A$ by $\mbf A(i, j)$, and index the $i$-th entry of a vector $\mbf e$ by $e(i)$.

\subsection{Default Representation}
We consider linearly solvable Markov decision processes~\citep[LS-MDPs;][]{todorov2006linearly, todorov2009efficient} $(\mathcal{S}, \mathcal{A}, r, p, \pi_d)$, where $\mathcal{S}$ is the state space, $\mathcal{A}$ is the action space, $r:\mathcal{S} \to \mathbb{R}_{< 0}$ is the reward function, $p: \mathcal{S} \times \mathcal{A} \to \Delta(\mathcal{S})$ is the transition function, and $\pi_d$ is the default policy. Let $\mathcal{S}_T \subset \mathcal{S}$ be the set of terminal states. We assume $\mathcal{S}_T$ is non-empty. We assume terminal states to be absorbing states and have a negative reward (see Appendix~\ref{appendix:formulation_diff} for a detailed discussion).

Let $\mbf r \in \mathbb{R}_{< 0}^{|\mathcal{S}|}$ be the vector of state rewards, and $\mbf{P}^{\pi_d} \in \mathbb{R}^{|\mathcal{S}| \times |\mathcal{S}|}$ be the transition probability matrix induced by $\pi_d$ and $p$. The default representation~\citep[DR;][]{piray2021linear, tse2025rewardaware}, defined with respect to $\pi_d$ and denoted by $\mbf Z$, is $\mbf Z = [\operatorname{diag}(\exp(- \mbf r / \lambda)) - \mathbf{P}^{\pi_d}]\inv$, where $\lambda>0$ is a hyperparameter. We use $\mbf R$ to denote $\operatorname{diag}(\exp(- \mathbf{r} / \lambda))$, and omit $\pi_d$ as it is clear from context, so we write $\mbf Z = (\mbf R - \mbf P)\inv$. In practice, we symmetrize the DR by replacing $\mbf P$ with $\tilde {\mbf P } = (\mbf P + \mbf P^\top) / 2$ to ensure a real eigensystem. 

The DR has been applied in a variety of settings. In particular, its principal eigenvector (corresponding to the largest eigenvalue), $\mbf e$, has been used for reward shaping and option discovery.
However, the entries of $\mbf e$ can have very small magnitudes, limiting its practical use without additional transformations. For example, simply applying $\mbf e$ as a potential function for potential-based reward shaping~\citep{ng1999policy} gives a shaping reward that is negligible almost everywhere.
To address this, \citet{tse2025rewardaware} proved, for the discrete setting, that the principal eigenvector of the symmetrized DR is positive, and, in their experiments, used $\log \mbf e$ instead, where $\log$ here denotes entry-wise logarithm.
Similarly, in this work, {\it we are interested in the logarithm of the principal eigenvector of the default representation}. \looseness=-1

\subsection{Graph Drawing Objective}
Spectral graph drawing~\citep{koren2003spectral} is a class of methods for visualizing graphs using the eigenvectors of matrices associated with the graph.
These methods have been applied in RL since an MDP can be viewed as a graph, where the vertices are the states, and the edge weights are the transition probabilities between states induced by $p$ and some policy~$\pi$. Notably, the graph drawing objective (GDO) and its variants have been applied to retrieve the eigenvectors of the Laplacian, $\mbf L = \mbf I - \tilde{\mbf P}^\pi$~\citep{wulaplacian, wang2021towards, gomez2024proper}. Specifically, we can retrieve the smallest eigenvector (corresponding to the smallest eigenvalue)\footnote{Note that this is equivalent to retrieving the principal eigenvector of the matrix's inverse, as is the case for the DR.} of a matrix $\mbf M \in \mathbb{R}^{|\mathcal{S}| \times |\mathcal{S}|}$ through the following objective: \looseness=-1
\begin{align}
    \min_{\mbf u \in \mathbb{R}^{|\mathcal{S}| \times |\mathcal{S}|}} \mbf u^\top \mbf M \mbf u \quad \text{s.t.}\quad \mbf u^\top \mbf u = c, \label{eq:gdo}
\end{align}
where $\mbf u$ is the learned eigenvector, and $c$ is a constant. In this work, we adapt the GDO for learning the logarithm of the principal eigenvector of the DR. \looseness=-1

\subsection{Natural Gradient}
Natural gradient descent~\citep{amari1998natural, amari1998natural2} improves convergence properties over standard gradient descent by adjusting the update direction to account for the underlying structure of the parameter space. 

Let $\mathcal{P} = \{ \mbf w \in \mathbb{R}^n \}$ be a parameter space, 
and the optimization problem be
$\min_{\mbf w} \ell(\mbf w)$, where $\ell$ is some loss function.
Standard gradient descent makes the simplifying assumption that the distance between two points in $\mathcal{P}$ is Euclidean. The descent direction, $\delta \mbf w$, is determined by the following optimization problem, where $\epsilon$ denotes some small constant:
\begin{align}
    \min_{\delta \mbf w } \ell(\mbf w) + \nabla \ell(\mbf w)^\top \delta \mbf w\quad \text{s.t.} \quad \|\delta\mbf w\|_2^2 = \epsilon^2.
\end{align}
That is, we search for the descent direction $\delta \mbf w$ that locally minimizes $\ell$, under the constraint that the Euclidean distance between the updated and original parameters is small. The descent direction $\delta \mbf w$ is then proportional to $-\nabla \ell(\mbf w)$.

Natural gradient descent considers a more general notion of distance, where the distance between $\mbf w$ and $\mbf w + \delta \mbf w$ is given by $\sqrt{\delta\mbf w^\top \mbf G(\mbf w) \delta \mbf w}$. $\mbf G(\mbf w)$ is known as the Riemannian metric tensor, which captures the local curvature of $\mathcal{P}$. Note that the Euclidean distance is a special case when $\mbf G = \mbf I$. The descent direction is now determined by
\begin{align}
    \min_{\delta \mbf w } \ell(\mbf w) + \nabla \ell(\mbf w)^\top \delta \mbf w\ \  \text{s.t.} \ \  \delta \mbf w^\top \mbf G(\mbf w) \delta \mbf w = \epsilon^2.
\end{align}
While we still constrain the distance between the updated and original parameters, the distance now reflects the local structure of the parameter space. The descent direction is then proportional to $-\mbf G(\mbf w)\inv \nabla \ell(\mbf w)$. 

\section{DROGO}
Prior work~\citep{tse2025rewardaware} adopts a computationally expensive procedure to obtain the logarithm of the principal eigenvector of the DR, $\log \mbf e$. Such an approach does not scale to high-dimensional spaces.
To address this, we introduce an objective for directly learning, using only transitions, the logarithm of $\mbf e$. We slowly build up our objective through theory developed in several sections: Section~\ref{sec:gdo} describes adapting the GDO for the DR. Section~\ref{sec:log_param} describes log-space parameterization to directly learn $\log \mbf e$. Section~\ref{sec:natural_grad} describes applying natural gradient to account for the log-space parameterization. Section~\ref{section:anchor} introduces an alternative constraint and provides theoretical justification. We conclude by introducing the final objective in Section~\ref{sec:drogo}.
\looseness=-1

\subsection{Graph Drawing Objective}
\label{sec:gdo}
Recall that we are interested in learning the {\it logarithm} of the principal eigenvector, $\log \mbf e$, of the symmetrized DR, $(\mbf R - \symP)\inv$. A naive approach is to first learn $\mbf e$ using the GDO, and then apply the logarithm. This is a necessary but not sufficient step towards our final objective, as we discuss below. \looseness=-1

To construct the GDO for the DR, we have:
\begin{proposition}
  \label{prop:1}
  The principal eigenvector of $(\mbf R - \tilde{\mbf P})\inv$ is equivalent to the smallest eigenvector of $\mbf I + \mbf R - \symP$.
\end{proposition}
\begin{proof}
    Since a matrix and its inverse share the same eigenvectors in reverse order, the principal eigenvector of $(\mbf R - \symP)\inv$ is equal to the smallest eigenvector of $\mbf R - \symP$. Furthermore, adding $\mbf I$ to $\mbf R - \symP$ only shifts all eigenvalues by 1, and does not affect the set of eigenvectors or their order.
\end{proof}

By Proposition~\ref{prop:1}, for some constant $c$, we can retrieve $\mbf e$ by learning the smallest eigenvector of $\mbf I + \mbf R - \symP$ using GDO: \looseness=-1
\begin{align}
    \min_{\mbf u \in \mathbb{R}^{|\mathcal{S}|}} \mbf u^\top (\mbf I + \mbf R - \symP) \mbf u \quad \text{s.t.} \quad \mbf u^\top \mbf u = c. \label{eq:dr_gdo}
\end{align}
By inspecting Eq.~\ref{eq:dr_gdo}, we gain intuition on what the principal eigenvector of the DR is in fact capturing. Note that $\mbf u^\top (\mbf I + \mbf R - \symP) \mbf u = \mbf u^\top (\mbf I - \symP)\mbf u + \mbf u^\top \mbf R\mbf u$, where the first term is just the objective for learning an eigenvector of the Laplacian (see Eq.~\ref{eq:gdo}), and the second term $\mbf u^\top \mbf R \mbf u = \sum_s \exp(-r(s) / \lambda) u(s)^2$ can be viewed as a regularization term. Recall that $r(s)$ is negative. If a state $s$ has a very negative reward, $\exp(-r(s) / \lambda)$ is large, and $u(s)$ is driven to have a small magnitude. The principal eigenvector of the DR can then be viewed as performing reward-aware regularization on top of an eigenvector of the Laplacian, where low-reward states are driven to have small magnitudes. Our empirical observations match this view.

As in common in prior work~\citep{wulaplacian, wang2021towards, gomez2024proper}, the constraint $\mbf u^\top \mbf u = c$ can be turned into a quadratic penalty term, giving
\begin{align}
    \min_{\mbf u \in \mathbb{R}^{|\mathcal{S}| }} \mbf u^\top (\mbf I + \mbf R - \symP) \mbf u + b(\mbf u^\top \mbf u - c)^2, \label{eq:dr_gdo_quad}
\end{align}
where $b$ is a hyperparameter. We denote the above objective function by $L_\text{GDO}(\mbf u)$.

We now derive the loss function for minimizing Eq.~\ref{eq:dr_gdo_quad}. 
\begin{proposition}
\label{prop:gdo_summation}
    $L_\text{GDO}(\mbf u)$ can be written as
    \begin{align}
    & L_\text{GDO}(\mbf u) = \frac{1}{2}\sum_{s}\sum_{s'}\symP(s, s') \big(u(s) - u(s') \big)^2 \nonumber
    \\ & + \sum_s \exp(-r(s) / \lambda) u(s)^2 + b \left ( \sum_{s}u(s)^2 - c
    \right )^2. \label{eq:dr_gdo_quad_summation}
\end{align}
\end{proposition}
\begin{proof}
    \begin{align}
        L_\text{GDO}(\mbf u) 
        &=\mbf u^\top (\mbf I - \symP) \mbf u + \mbf u^\top \mbf R \mbf u + b(\mbf u^\top \mbf u - c)^2.
    \end{align}
    It is known that the first term can be written as~\citep{koren2003spectral}
    \begin{align}
        \mbf u^\top (\mbf I - \symP) \mbf u = \frac{1}{2}\sum_{s}\sum_{s'}\symP(s, s') \big(u(s) - u(s') \big)^2.
    \end{align}
    Finally, we rewrite $\mbf u^\top \mbf R \mbf u$ and $\mbf u^\top \mbf u$ as summations.
\end{proof}

Let $u_{\boldsymbol \theta}$ be the function approximated by a neural network parameterized by ${\boldsymbol \theta}$, and $\mbf u_{\boldsymbol \theta} = [u_{\boldsymbol \theta}(s_1), \dots,u_{\boldsymbol \theta}(s_{|\mathcal{S}|})]^\top$ be the vector predicted by the neural network. We have:
\begin{proposition} \label{prop:gdo_loss}
Given a transition $(s, a, r, s')$ and an extra state $s''$, where $s$ and $s''$ are sampled independently and uniformly at random, and $a$ is sampled under the default policy, $\pi_d$, we can minimize $ L_\text{GDO}(\mbf u_{\boldsymbol{\theta}})$ by minimizing 
\begin{align}
    & \ell_{\text{GDO}}({\boldsymbol \theta}) =  \frac{1}{2} \big( u_{\boldsymbol \theta}(s) - u_{\boldsymbol \theta}(s') \big )^2 + \exp(-r / \lambda) u_{\boldsymbol \theta}(s)^2  \nonumber
    \\ & \ \ \ \ \ \ \ \ \ \ \ \ \ \ \ \ \ \ \ \ \ \ + b' (u_{\boldsymbol \theta}(s)^2 - c') (u_{\boldsymbol \theta}(s'')^2 - c'), \label{eq:dr_gdo_quad_loss}
\end{align}
    where $b'$ and $c'$ can be treated as hyperparameters.
\end{proposition}
{\it Proof sketch.} We express Eq.~\ref{eq:dr_gdo_quad_summation} in terms of expectations. $\ell_\text{GDO}$ is simply a sample-based approximation of the expectations. See Appendix~\ref{appendix:proof_of_prop_gdo_loss} for full proof. 

We can compute $\ell_\text{GDO}$ using sampled transitions, then perform backpropagation and gradient descent to update ${\boldsymbol \theta}$. However, recall that some entries of $\mbf e$ can have very small magnitudes. By optimizing Eq.~\ref{eq:dr_gdo_quad_loss}, we observe that even small approximation errors can cause the neural network's prediction, $u_{\boldsymbol \theta}(s)$, to be negative for some low-reward states. This prohibits us from taking the logarithm and using the learned eigenvector in practice.
While increasing $\lambda$ can reduce the strength of the reward-aware regularization, extremely large $\lambda$s are required to fully address this problem. However, this comes at the expense of the reward-awareness of $\mbf e$, as $\mbf R - \symP$ approaches the Laplacian, $\mbf I - \symP$, as $\lambda \to \infty$.

\subsection{Log-Space Parameterization}
\label{sec:log_param}
To guarantee that the learned eigenvector is always positive, while ensuring that it retains certain levels of reward-awareness, we perform log-space parameterization. Specifically, we directly approximate $\log \mbf e$ using a neural network. 
First, Eq.~\ref{eq:dr_gdo_quad} can be adapted as
\begin{align}
    \min_{\mbf v \in \mathbb{R}^{|\mathcal{S}|}} \exp(\mbf v)^\top (\mbf I + \mbf R - \symP) \exp(\mbf v) + \nonumber  \\ 
    b(\exp(\mbf v)^\top \exp(\mbf v) - c)^2 , \label{eq:gdo_quad_log}
\end{align}
where $\exp$ here denotes element-wise exponentiation. 
We denote the above objective function by $ L^{\log}_\text{GDO}(\mbf v)$. Note that $ L^
{\log}_\text{GDO}(\mbf v) = L_\text{GDO}(\exp(\mbf v))$.

Let $v_{\boldsymbol \phi}$ be the function approximated by a neural network parameterized by ${\boldsymbol \phi}$. We define $\mbf v_{\boldsymbol{\phi}}$ similar to how we defined $\mbf u_{\boldsymbol{\theta}}$, and define $\mbf u_{\boldsymbol \phi} = \exp(\mbf v_{\boldsymbol \phi})$ to be the element-wise exponentiation of $\mbf v_{\boldsymbol{\phi}}$.
Building on Propositions~\ref{prop:gdo_summation} and~\ref{prop:gdo_loss}: 
\begin{proposition}
    Given a transition $(s, a, r, s')$ and an extra state $s''$, where $s$ and $s''$ are sampled independently and uniformly at random, and $a$ is sampled under the default policy, $\pi_d$, we can minimize $ L^{\log}_\text{GDO}(\mbf v_{\boldsymbol{\phi}})$ by minimizing
    \begin{align}
    \ell^{\log}_{\text{GDO}}({\boldsymbol \phi}) &=  \frac{1}{2} \big( \exp(v_{\boldsymbol \phi}(s)) - \exp(v_{\boldsymbol \phi}(s')) \big )^2  \nonumber
    \\ & \ \ \ \ \ \ \ + b (\exp(2 v_{\boldsymbol \phi}(s)) - c) (\exp(2v_{\boldsymbol \phi}(s'')) - c) \nonumber
    \\ & \ \ \ \ \ \ \ + \exp(2v_{\boldsymbol \phi}(s) -r / \lambda).  \label{eq:dr_gdo_quad_log_loss}
\end{align}
\end{proposition}
\begin{proof}
    We replace $u_{\boldsymbol \theta}(s)$ with $\exp(v_{\boldsymbol \phi}(s))$ in Eq.~\ref{eq:dr_gdo_quad_loss}.
\end{proof}

While it is possible to just approximate ${\ell}^{\log}_\text{GDO}$ and perform backpropagation and gradient descent to update ${\boldsymbol \phi}$, we perform further analysis on the gradient of $L^{\log}_\text{GDO}(\mbf v_{\boldsymbol{\phi}})$. By comparing the gradient of $ L_\text{GDO}(\mbf u_{\boldsymbol{\theta}})$ and $ L^{\log}_\text{GDO}(\mbf v_{\boldsymbol{\phi}})$, we reveal a source of instability for the loss function, $\ell^{\log}_\text{GDO}$.

Using the multivariate chain rule, the gradient of $ L_\text{GDO}(\mbf u_{\boldsymbol{\theta}})$ with respect to the $i$-th entry of the parameters, $\boldsymbol \theta$, is
\begin{align}
    \frac{\partial L_\text{GDO}(\mbf u_{\boldsymbol{\theta}})}{\partial \theta_{i}} =\frac{\partial \mbf u_{\boldsymbol \theta}}{\partial \theta_{i}} \cdot \frac{\partial L_\text{GDO}(\mbf u_{\boldsymbol{\theta}})}{\partial \mbf u_{\boldsymbol \theta}}, \label{eq:dr_gdo_quad_loss_grad}
\end{align}
while the gradient of $ L^{\log}_\text{GDO}(\mbf v_{\boldsymbol{\phi}})$ w.r.t. the $i$-th entry of $\boldsymbol{\phi}$ is

\begin{align}
    \frac{\partial  L^{\log}_\text{GDO}(\mbf v_{\boldsymbol{\phi}})}{\partial \phi_{i}} &=\frac{\partial \mbf v_{\boldsymbol \phi}}{\partial \phi_{i}} \cdot \frac{\partial  L^{\log}_\text{GDO}(\mbf v_{\boldsymbol{\phi}})}{\partial \mbf v_{\boldsymbol \phi}} \\
    &= \frac{\partial \mbf v_{\boldsymbol \phi}}{\partial \phi_{i}} \cdot 
    \left ( \exp(\mbf v_{\boldsymbol \phi}) \odot
    \frac{\partial  L^{\log}_\text{GDO}(\mbf v_{\boldsymbol{\phi}})}{\partial \mbf u_{\boldsymbol \phi}} 
    \right ) \\
    &= \frac{\partial \mbf v_{\boldsymbol \phi}}{\partial \phi_{i}} \cdot 
    \left ( \exp(\mbf v_{\boldsymbol \phi}) \odot
    \frac{\partial L_\text{GDO}(\mbf u_{\boldsymbol{\phi}})}{\partial \mbf u_{\boldsymbol \phi}} 
    \right ) \label{eq:dr_gdo_quad_log_loss_grad}
\end{align}
where $\odot$ is the Hadamard product. Comparing Eq.~\ref{eq:dr_gdo_quad_loss_grad} and~\ref{eq:dr_gdo_quad_log_loss_grad}, we see that the log-space parameterization introduces an extra multiplicative exponential term $\exp(\mbf v_{\boldsymbol \phi})$, which scales exponentially with the network output, and can lead to instability in the optimization process. Next, we apply natural gradient to account for this exponential term. \looseness=-1

\subsection{Natural Gradient}
\label{sec:natural_grad}

When performing log-space parameterization, we essentially move from solving $\min_{\mbf u} L(\mbf u)$ to solving $\min_{\mbf v}  L(\exp(\mbf v))$, and now we have $\mbf v$ as parameters of the optimization problem instead of $\mbf u$. When performing standard gradient descent, we are implicitly assuming that the difference between the original and updated parameters is measured by their Euclidean distance, or equivalently, the size of the gradient step is measured by its Euclidean norm.
While this assumption generally works without issues, assuming the Euclidean distance can be undesirable when we perform log-space parameterization.
Suppose we take a gradient step with a tiny Euclidean norm.
When $\mbf v$ has large values, the gradient step causes massive changes in $\exp(\mbf v)$, and thus, $L(\exp(\mbf v))$. When $\mbf v$ has small values, the same gradient step causes negligible changes. 
This is the reason we observe the exponential term, $\exp(\mbf v_{\boldsymbol{\phi}})$, in Eq.~\ref{eq:dr_gdo_quad_log_loss_grad}. To summarize, the $\mbf v$-space's local geometry is inherently different from that of the $\mbf u$-space. \looseness=-1

We can address this problem with natural gradient, by explicitly accounting for the difference in the geometry of the $\mbf v$-space. To do so, we need to identify the distance in the $\mbf v$-space induced by: (1) the relationship between the $\mbf v$- and $\mbf u$-spaces, and (2) the Euclidean assumption on the $\mbf u$-space (see Figure~\ref{fig:nat_grad}, in the Appendix, for a visual illustration). 
\begin{proposition}
    The distance in the $\mbf v$-space is characterized by the Riemannian metric tensor $\mbf G(\mbf v) = \diag(\exp(2\mbf v))$. \looseness=-1 \label{prop:riemannian}
\end{proposition}
{\it Proof sketch.} Let $\delta \mbf u$ be the change induced in $\mbf u$ by $\delta \mbf v$, a small step in the $\mbf v$-space. We identify $\mbf G(\mbf v)$ by solving \looseness=-1
\begin{align}
    d_{\mathcal{U}}(\mbf u, \mbf u + \delta \mbf u)^2 = d_{\mathcal{V}}(\mbf v, \mbf v + \delta \mbf v)^2,
\end{align}
where $d_\mathcal{U}$ and $d_\mathcal{V}$ are distances in the $\mbf u$- and $\mbf v$-spaces respectively. We provide the full proof in Appendix~\ref{appendix:riemannian}.

\begin{corollary}
    Let $-\frac{\partial L}{\partial \mbf v}$ be the standard gradient descent direction. The adjusted natural gradient descent direction is simply $-\mbf G(\mbf v)\inv \frac{\partial L}{\partial \mbf v} = -\operatorname{diag}(\exp(-2\mbf v)) \frac{\partial L}{\partial \mbf v}$.
\end{corollary}

We apply the natural gradient to minimize $ L^{\log}_\text{GDO}(\mbf v_{\boldsymbol{\phi}})$. Incorporating the Riemannian metric tensor to Eq.~\ref{eq:dr_gdo_quad_log_loss_grad}, the gradient of $ L^{\log}_\text{GDO}(\mbf v_{\boldsymbol{\phi}})$ w.r.t. the $i$-th entry of the parameters, $\boldsymbol{\phi}$, is \looseness=-1
\begin{align}
    & \placeholder \frac{\partial \mbf v_{\boldsymbol \phi}}{\partial \phi_i} \cdot \mbf G(\mbf v_{\boldsymbol \phi})\inv \left(
    \exp(\mbf v_{\boldsymbol \phi}) \odot  \frac{\partial  L_{\text{GDO}}(\mbf u_{\boldsymbol{\phi}})}{\partial \mbf u_{\boldsymbol \phi}}
    \right ) \\
    &= \frac{\partial \mbf v_{\boldsymbol \phi}}{\partial \phi_i} \cdot \left(
    \exp(- \mbf v_{\boldsymbol \phi}) \odot  \frac{\partial L_{\text{GDO}}(\mbf u_{\boldsymbol{\phi}})}{\partial \mbf u_{\boldsymbol \phi}}
    \right ).  \label{eq:dr_quad_nat_grad}
\end{align}

Note that Eq.~\ref{eq:dr_gdo_quad_log_loss_grad} and~\ref{eq:dr_quad_nat_grad} differ only by the sign inside the exponential. Now, $\exp(-\mbf v_{\boldsymbol{\phi}})$ ensures that updates to $\mbf v_{\boldsymbol{\phi}}$ produce similar effects on $ L^{\log}_\text{GDO}(\mbf v_{\boldsymbol{\phi}})$ regardless of the value $\mbf v_{\boldsymbol{\phi}}$ takes. For example, if $v_{\boldsymbol{\phi}}(s)$ is very large, for which even a tiny update can drastically change $ L^{\log}_\text{GDO}(\mbf v_{\boldsymbol{\phi}})$, $\exp(-\mbf v_{\boldsymbol{\phi}})$ scales the update to $v_{\boldsymbol{\phi}}(s)$ by $1 / \exp(v_{\boldsymbol{\phi}}(s))$. \looseness=-1

Using the above equation, we now derive a new loss function that incorporates the natural gradient. Note that Eq.~\ref{eq:dr_quad_nat_grad} is the dot product of two terms, where the first term is the gradient of the neural network's output w.r.t. its parameters, and the second term is the natural gradient of the objective function, $ L^{\log}_\text{GDO}(\mbf v_{\boldsymbol{\phi}})$, with respect to the neural network's output. We first present two propositions for the second term:
\begin{proposition} \label{prop:nat_grad_second_term}
    The natural gradient of $ L^{\log}_\text{GDO}(\mbf v_{\boldsymbol{\phi}})$ with respect to the output of the neural network parameterized by $\boldsymbol{\phi}$, $\mbf v_{\boldsymbol{\phi}}$, is proportional to
    \begin{align}
        \exp(-\mbf r / \lambda) -  \exp(-\mbf v_{\boldsymbol \phi}) \odot \symP \mbf u_{\boldsymbol \phi} + b'(\mbf u_{\boldsymbol \phi}^\top \mbf u_{\boldsymbol \phi} - c') \mbf 1. \label{eq:dr_gdo_quad_log_grad}
    \end{align}
Proof. \textnormal{See Appendix~\ref{appendix:nat_grad_second_term_proof}.}\hfill $\square$
\end{proposition}

\vspace{0.2cm}

\begin{proposition}
    \label{prop:s_th_entry_estimate}
    The $s$-th entry of Eq.~\ref{eq:dr_gdo_quad_log_grad} is estimated by \looseness=-1
    \begin{align}
    \exp(-r/\lambda) - \exp(v_{\boldsymbol \phi}(s') - v_{\boldsymbol \phi}(s)) \nonumber  \\ + b''(\exp(2v_{\boldsymbol \phi}(s'')) - c''),
\end{align}
where $a$ in $(s, a, r, s')$ is sampled from the default policy, $\pi_d$, and $s''$ is a state sampled uniformly at random.

Proof. \textnormal{See Appendix~\ref{appendix:s_th_entry_proof}.}\hfill $\square$
\end{proposition}

\vspace{0.2cm}

\begin{proposition} \label{prop:surrogate_loss}
Given a transition $(s, a, r, s')$ and an extra state $s''$, where $s$ and $s''$ are sampled independently and uniformly at random, and $a$ is sampled under the default policy, $\pi_d$, we can minimize $ L^{\log}_\text{GDO}(\mbf v_{\boldsymbol{\phi}})$, with hyperparameters $\lambda, b'', c''$, with natural gradient descent by minimizing 
\begin{align}
    \ell^{\text{NG}}_{\text{GDO}}({\boldsymbol \phi}) = \llbracket \exp(-r / \lambda) - \exp(v_{\boldsymbol \phi}(s') - v_{\boldsymbol \phi}(s)) \nonumber \\
     + b''(\exp(2v_{\boldsymbol \phi}(s'')) - c'') \rrbracket v_{\boldsymbol \phi}(s), \label{eq:dr_quad_surr_loss}
\end{align}
where $\llbracket \cdot \rrbracket$ denotes the stop-gradient operator. \looseness=-1
\end{proposition}
{\it Proof sketch.} Using Propositions~\ref{prop:nat_grad_second_term} and~\ref{prop:s_th_entry_estimate}, we derive a sampled-based approximation of $\frac{\partial  L^{\log}_\text{GDO}(\mbf v_{\boldsymbol{\phi}})}{\partial \boldsymbol{\phi}}$. It is trivial to show that the gradient of $ \ell^{\text{NG}}_\text{GDO}$ recovers this approximation. We provide the full proof in Appendix~\ref{appendix:surrogate_loss_proof}.


When using $\ell^{\text{NG}}_\text{GDO}$, we are able to successfully learn $\log \mbf e$ when using one-hot representations of states as input to the neural network. However, we observe instabilities during optimization when using other state representations, which stems from the difficulty in balancing the GDO, $\exp(\mbf v)^\top (\mbf I + \mbf R - \symP) \exp(\mbf v)$, and the quadratic norm penalty, $(\exp(\mbf v)^\top \exp(\mbf v) - c)^2$. In $\ell^{\text{NG}}_{\text{GDO}}({\boldsymbol \phi})$, $b''$ is the hyperparameter balancing the two terms. We observe empirically that when $b''$ is too small, the norm penalty is not strong enough, and the norm of the learned $\exp(\mbf v_{\boldsymbol{\phi}})$ can decrease to near 0. When $b''$ is too large, the norm penalty constrains the norm of $\exp(\mbf v_{\boldsymbol{\phi}})$ well, but interferes with the objective. Dynamically adapting $b''$ offers little improvement. \looseness=-1

In the next section, we introduce the final adaptation to obtain our final loss function. It differs from $ \ell^{\text{NG}}_\text{GDO}$ by how we constrain the learned vector, given our observation that the quadratic norm penalty can be difficult to balance.  \looseness=-1

\subsection{Alternative Constraint on the Learned Vector}
\label{section:anchor}

Despite being difficult to balance, the quadratic norm penalty prevents the learned eigenvector from collapsing to the zero vector. An alternative way of seeing this is to imagine that, without the quadratic norm penalty, a learned vector perfectly aligned with the ground-truth can be expressed as $\alpha \mbf e$ for some $\alpha$, and 
in this case the objective function is minimized when $\alpha \to 0$. We still have the same issue when learning $\log \mbf e$ directly instead of $\mbf e$, as multiplying $\mbf e$ by a positive constant $\alpha$ would be equivalent to adding the constant $\log \alpha$ to all entries of $\mbf e$. When $\alpha \to 0$, all entries of the solution keep drifting to the left on the real number line. \looseness=-1

Thus, to drop the quadratic norm penalty, we need to prevent such constant shifts. We do so by fixing the value of a particular entry of the learned vector. In particular, we fix the value of a selected terminal state, $v(s_T)$, to 0. Intuitively, $v(s_T)$ acts as an anchor that holds the entire vector $\mbf v$ in place.
Mathematically, for a terminal state $s_T \in \mathcal{S}_T$, we solve the following optimization problem: 
\begin{align}
    \min_{\mbf v \in \mathbb{R}^{|\mathcal{S}|}} \exp(\mbf v)^\top (\mbf R - \symP) \exp(\mbf v) \ \text{s.t.}\ v(s_T) = 0. \label{eq:dr_anchor}
\end{align} 


While Eq.~\ref{eq:dr_anchor} might make intuitive sense, it is in fact not trivial to show that by minimizing this optimization problem instead, we still approximate the smallest eigenvector of $\mbf R - \symP$. We first study the solution of Eq.~\ref{eq:dr_anchor}, and present:
\begin{proposition}
    Minimizing Eq.~\ref{eq:dr_anchor} retrieves the logarithm of the column of the DR corresponding to the selected terminal state, $s_T$.
\end{proposition}
\begin{proof}
    Using Lagrange multipliers, Eq.~\ref{eq:dr_anchor} is equivalent to
\begin{align}
    \min_{\mbf v \in \mathbb{R}^{|\mathcal{S}|}} \max_{\mu \in \mathbb{R}} \exp(\mbf v)^\top (\mbf R - \symP) \exp(\mbf v) + \mu  v(s_T).
\end{align}
Let $\mbf b_i$ be the standard basis vector with a 1 at entry $i$. Taking the gradient with respect to $\mbf v$ and setting it to 0, we have
\begin{align}
    2\exp(\mbf v) \odot (\mbf R - \symP)\exp(\mbf v) + \mu \mbf b_{s_T} = 0 \\
    (\mbf R - \symP)\exp(\mbf v) = -\frac{\mu}{2} \exp(-v(s_T)) \mbf b_{s_T} \\
    \exp(\mbf v) = -\frac{\mu}{2} (\mbf R - \symP)\inv \mbf b_{s_T} 
\end{align}
where $(\mbf R - \symP)\inv \mbf b_{s_T}$ is just the column of the DR corresponding to the chosen terminal state, $s_T$.
\end{proof}

Theoretical analysis reveals that by solving Eq.~\ref{eq:dr_anchor}, we in fact retrieve the $s_T$ column of the DR. However, we observe empirically that the vector we learn aligns with the principal eigenvector of the DR very well. We remedy this seeming contradiction by proving the following theorem:
\begin{theorem}
  \label{thm:eigenspace}
  Let $r(s_T) = -\delta\ \forall s_T \in \mathcal{S}_T$, where $\delta > 0$. When $\delta\to0^+$, the columns of the DR ($(\mbf R - \mbf P)\inv$) corresponding to all terminal states form a basis for the eigenspace of the principal eigenvalue of the DR.
\end{theorem}

{\it Proof sketch.} We first show that $\delta$ is the principal eigenvalue of the DR. We derive the DR columns corresponding to terminal states, and the eigenvectors corresponding to $\delta$. We conclude the theorem by comparing their formulae. We provide the full proof in Appendix~\ref{appendix:eigenspace}.

Note that Theorem~\ref{thm:eigenspace} considers the DR without symmetrization, while in practice, such as in Eq.~\ref{eq:dr_anchor}, we use the symmetrized DR, where we replace $\mbf P$ with $\tilde{\mbf P}=(\mbf P + \mbf P^\top ) / 2$. Nevertheless, Theorem~\ref{thm:eigenspace} provides good intuition for why optimizing Eq.~\ref{eq:dr_anchor} retrieves a vector that highly aligns with the principal eigenvector of the symmetrized DR. \looseness=-1

\begin{remark}
    Theorem~\ref{thm:eigenspace}, combined with the fact that the $s$-th column of the DR captures the reward-aware distance from all states to $s$~\citep{tse2025rewardaware}, explains why the principal eigenvector of the DR has been successful in applications such as reward shaping. \looseness=-1
\end{remark}


\subsection{DROGO}
\label{sec:drogo}
We present the \textbf{D}efault \textbf{R}epresentation \textbf{O}bjective via \textbf{G}raph \textbf{O}ptimization (DROGO), which approximates 
$\log \mbf e$ by optimizing Eq.~\ref{eq:dr_anchor} using natural gradient. Let $v_{\boldsymbol \phi}$ be the function approximated by a neural network parameterized by ${\boldsymbol \phi}$. \looseness=-1 

\begin{theorem}[DROGO loss function]
    \label{thm:drogo_loss}
    Given a transition $(s,a, r, s')$, where $s$ is sampled uniformly at random, and $a$ is sampled under the default policy, $\pi_d$, we minimize Eq.~\ref{eq:dr_anchor} by minimizing
\begin{align}
    &\ell_{\text{DROGO}}(\boldsymbol{\phi}) = \Big \llbracket \exp(-r / \lambda)  \nonumber \\ & - \exp\Big(v_{\boldsymbol \phi}(s') (1 - \mathbb{I}_{\{ s' = s_T \}} ) - v_{\boldsymbol \phi}(s) \Big) \Big \rrbracket v_{\boldsymbol \phi}(s). \label{eq:drogo_loss}
\end{align}
\end{theorem}
\begin{proof}
    Removing the quadratic penalty, Eq.~\ref{eq:dr_quad_surr_loss} becomes
    \begin{align}
        \llbracket \exp(-r / \lambda) - \exp(v_{\boldsymbol \phi}(s') - v_{\boldsymbol \phi}(s))  \rrbracket v_{\boldsymbol \phi}(s).
    \end{align}
    The constraint $v(s_T) = 0$ can be enforced simply by multiplying $v_{\boldsymbol \phi}(s')$ by $( 1 - \mathbb{I}_{\{ s' = s_T \}})$, where $\mathbb{I}$ denotes the indicator function.
\end{proof}

\begin{figure*}[ht]
  \vskip 0.2in
  \begin{center}
    \centerline{\includegraphics[width=0.9\textwidth]{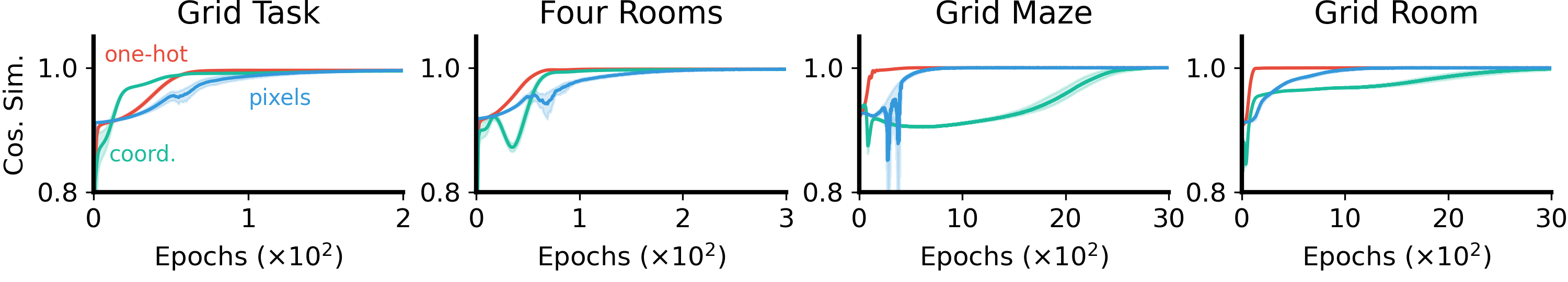}}
    \caption{
      The cosine similarity between the learned and ground-truth principal eigenvectors for different state representations averaged over 10 seeds. The shaded area indicates 95\% bootstrap confidence interval.
    }
    \label{fig:cos_sim}
  \end{center}
\end{figure*}

\section{Experiments}
We verify empirically that optimizing the DROGO objective retrieves the principal eigenvector of the DR. We present an ablation study in Appendix~\ref{appendix:ablation}. Finally, we show the utility of the learned eigenvectors for reward shaping. \looseness=-1

\subsection{Learning the Principal Eigenvector of the DR}
\label{section:cos_sim_exp}

We apply DROGO to approximate the principal eigenvector of the DR using neural networks. To be able to validate our results against the ground-truth, we consider the set of grid world environments used by \citet{tse2025rewardaware} (see Appendix~\ref{appendix:environments}). A neural network takes as input a state, and outputs the logarithm of the corresponding entry in the principal eigenvector. 

We generate 200,000 transitions using a uniform random initial state distribution and the default policy, which we choose to be the uniform random policy. We then train the neural networks by minimizing the DROGO loss (Eq.~\ref{eq:drogo_loss}) using mini-batches from the dataset. To demonstrate the robustness of DROGO, we use three different representations of states as input to the neural network: (1) one-hot encoding, where state $s_i$ is represented as $\mbf b_i$, (2) $(x, y)$-coordinates, and (3) pixels. See details in Appendix~\ref{appendix:drogo_details}. \looseness=-1

To evaluate the quality of the approximation, we compute the cosine similarity between the learned vector and the ground-truth. The cosine similarity is computed between the logarithm of these eigenvectors. 
Fig.~\ref{fig:cos_sim} shows the cosine similarity in four grid world environments for the three state representations. We can see that regardless of the choice of the state representation, DROGO enables approximating the principal eigenvector very well, as the cosine similarity converges to near 1 for all settings. A visual comparison of the learned and ground-truth eigenvectors is in Appendix~\ref{appendix:drogo_details}.

\subsection{Reward Shaping}

While we demonstrate that the cosine similarity between the learned and ground-truth principal eigenvectors is close to 1, this metric does not fully reflect the performance of using these learned eigenvectors. To demonstrate the practical utility of the learned eigenvectors, we perform reward shaping. \looseness=-1

In reward shaping, we provide the agent with a shaping reward to aid in learning the optimal policy. We apply the principal eigenvector of the DR for potential-based reward shaping~\citep{ng1999policy}. Let $\mbf v$ be the learned eigenvector, the shaping reward is $\tilde{r}_t = \gamma v(s_{t + 1}) - v(s_t)$, where $\gamma$ is the discount factor. Then, we train a Q-learning~\citep{watkins1989learning} agent using a convex combination of the original environment reward and the shaping reward, resulting in $(1 - \beta)r_t + \beta \tilde{r}_t$, where $\beta \in [0, 1]$ is a hyperparameter.

We perform reward shaping using the principal eigenvector of the DR in separate experiments for each state representation. We compare against using the principal eigenvector of the SR~\citep{dayan1993improving}, and no reward shaping (plain Q-learning). Applying the principal eigenvectors of the SR and DR for reward shaping is common in RL~\citep{wulaplacian, wang2021towards, tse2025rewardaware}. We compute the principal eigenvector of the SR in closed form to prevent approximation errors from hurting the performance of the baseline, effectively giving it an advantage. Similar to \citet{tse2025rewardaware}, we use $\gamma = 0.99$, $\epsilon = 0.05$ for $\epsilon$-greedy exploration, and perform a grid search over $\beta$ ($[0.25, 0.5, 0.75, 1.0]$), and Q-learning's step size ($[0.1,0.3,1.0]$). \looseness=-1

We perform reward shaping in the same environments of Section~\ref{section:cos_sim_exp}. In these environments, the agent is tasked with reaching the goal state while avoiding low-reward regions on the way.
To measure the effectiveness of different shaping rewards, we evaluate the mean number of steps required to converge to the optimal return ($N_\text{OPT}$). As the DR has been shown to give rise to reward-aware behavior, we also measure the mean number of visits to low-reward states during training ($N_\text{VISIT}$). The lower $N_\text{OPT}$, the more effective the shaping reward. The lower $N_\text{VISIT}$, the more the shaping reward guides the agent to avoid low-reward regions. \looseness=-1 

Table~\ref{table:reward_shaping} presents $N_\text{OPT}$ and $N_\text{VISIT}$ averaged over 10 seeds for different shaping rewards. In the grid room environment, the SR and no reward shaping cannot converge to the optimal return by the end of training, so the corresponding $N_\text{OPT}$'s are not reported. 
Using the DR for reward shaping consistently leads to lower $N_\text{OPT}$ compared to the baselines. 
We observe that when using the DR with pixels representation for reward shaping in four rooms, there are a small number of seeds with slower learning speed due to small approximation errors in the principal eigenvector, leading to larger $N_\text{OPT}$ than its one-hot and coordinates counterparts. Nevertheless, it achieves comparable performance to the SR, whose principal eigenvector is computed in closed form and does not have to contend with approximation errors.
On the other hand, using the DR for reward shaping leads to similar or lower $N_\text{VISIT}$ compared to the SR. The difference is particularly significant in four rooms and grid room.  \looseness=-1

In general, the DR performs better than the SR. This is because, unlike the SR, the DR captures reward information in the principal eigenvector. During reward shaping, the reward-aware principal eigenvector causes the agent to prefer high-reward paths in the environment, eventually leading to lower $N_\text{OPT}$ and $N_\text{VISIT}$. These results match the results reported by \citet{tse2025rewardaware} in the tabular case, but now using eigenvectors approximated through neural networks obtained from high-dimensional observations. This is a substantial improvement over computing such eigenvectors in closed form in the tabular case. Finally, we also provide learning curves in Appendix~\ref{appendix:reward_shaping}. \looseness=-1

\begin{table*}[t]
  \caption{The mean number of steps, over 10 seeds, required to converge to the optimal return ($N_\text{OPT}$) and the mean number of visits to low-reward states ($N_\text{VISIT}$) when using the DR with one-hot ($\text{DR}_{\textsc{one-hot}}$), coordinates ($\text{DR}_{\textsc{coord}}$), and pixels ($\text{DR}_{\textsc{pixels}}$) as state representation for reward shaping, using the SR for reward shaping (SR), and not performing reward shaping (NS). The lower- and superscripts denote the 95\% bootstrap confidence interval. We highlight all DR methods whose confidence intervals do not overlap with that of the SR. \looseness=-1}
  \label{table:reward_shaping}
  \begin{center}
    \begin{small}
      \begin{sc}
      \renewcommand{\arraystretch}{1.5}
        \resizebox{\textwidth}{!}{
        \begin{tabular}{lrrrrrrrrrr}
          \toprule
          Env  & $N_\text{OPT}$ $(\times 10^3)$  &  & & & & $N_\text{VISIT}$ $(\times 10^2)$ & &&  &  \\
          \cmidrule{2-11}
          & NS & SR & $\text{DR}_{\text{one-hot}}$ & $\text{DR}_{\text{coord}}$ & $\text{DR}_{\text{pixels}}$ & NS & SR & $\text{DR}_{\text{one-hot}}$ & $\text{DR}_{\text{coord}}$ & $\text{DR}_{\text{pixels}}$ \\
          \midrule
          Grid Task  & \tentry{13.75}{0.18}{0.17} & \tentry{2.04}{0.04}{0.04} & \btentry{0.97}{0.04}{0.04} & \btentry{0.98}{0.05}{0.04} & \btentry{0.98}{0.05}{0.05} & \tentry{7.87}{0.12}{0.12} & \tentry{6.51}{0.43}{0.49} & \tentry{6.62}{0.56}{0.55} & \tentry{6.61}{0.77}{0.80} & \tentry{6.60}{0.61}{0.69}   \\
          Four Rooms & \tentry{14.39}{0.19}{0.27} & \tentry{1.37}{0.04}{0.04} & \btentry{0.18}{0.03}{0.04} & \btentry{0.16}{0.00}{0.00} & \tentry{1.25}{1.08}{1.83} & \tentry{3.12}{0.08}{0.07} & \tentry{1.30}{0.03}{0.03} & \btentry{0.05}{0.03}{0.03} & \btentry{0.07}{0.03}{0.03} & \btentry{0.07}{0.05}{0.04}\\
          Grid Maze & \tentry{71.82}{3.32}{4.29} & \tentry{15.20}{0.34}{0.37} & \btentry{5.97}{0.28}{0.29} & \btentry{6.46}{0.72}{1.14} & \btentry{7.47}{1.45}{1.63} & \tentry{0.65}{0.03}{0.03} & \tentry{0.42}{0.03}{0.04} & \tentry{0.44}{0.03}{0.03} & \tentry{0.46}{0.03}{0.03} & \tentry{0.46}{0.03}{0.03}  \\
          Grid Room & - & - & \btentry{7.68}{0.17}{0.17} & \btentry{7.72}{0.32}{0.25} & \btentry{7.83}{0.24}{0.22} & \tentry{1.48}{0.08}{0.05} & \tentry{1.73}{0.08}{0.07} & \btentry{0.53}{0.03}{0.03} & \btentry{0.55}{0.04}{0.05} & \btentry{0.54}{0.03}{0.03} \\
          \bottomrule
        \end{tabular}
        }
      \end{sc}
    \end{small}
  \end{center}
  \vskip -0.1in
\end{table*}

\section{Related Work}

The SR~\citep{dayan1993improving} represents a state by the expected discounted number of visits to successor states. Specifically, it counts all visits to successor states. The first-occupancy representation~\citep[FR;][]{moskovitzfirst} is a variant that captures only the first visit to successor states. The $\lambda$ representation~\citep[$\lambda$R;][]{moskovitz2023state} can be viewed as an interpolation between the SR and FR. Compared to the FR and $\lambda$R, the DR~\citep{piray2021linear} generalizes the SR in an orthogonal direction by capturing not just the transitions, but also the rewards in the environment. 

The SR and its extension in continuous spaces, the successor measure~\citep{Blier2021learning}, have given rise to a wide range of methods for transfer. Such works include successor features~\citep{barreto2017successor, barreto2018transfer, BorsaBQMHMSS19}, forward-backward representation~\citep{touati2021learning, TouatiRO23, TirinzoniTFGKXL25}, proto-successor measure~\citep{agarwal2025proto}, and distributional successor measure~\citep{WiltzerFGT0DBR24}.
The DR, unlike the SR, also captures the reward function. Being more specific to environment dynamics comes at the price of less flexible transfer. Without assuming access to environment dynamics, it can be used to directly compute the optimal policy when only rewards at terminal states change~\citep{tse2025rewardaware}. 

Apart from transfer, where the goal is to efficiently compute the optimal policy for downstream rewards, the SR has been applied to better learn from a single reward function. 
The SR has the same set of eigenvectors as the Laplacian~\citep{machado2018eigenoption, stachenfeld2014design}, which was shown to be a good basis for approximating the value function~\citep{mahadevan2005proto, mahadevan2007proto}. As they reflect the geometry of the state space, they have been subsequently applied as a distance metric for reward shaping~\citep{wulaplacian, wang2021towards, wang2023reachability}, an inductive bias for temporally-extended exploration~\citep{machado2023temporal, klissarov2023deep}, effective state representations~\citep{le2022generalization, FarebrotherGALG23}, and a density model for count-based exploration~\cite{machado2020count}.
Given the close connection between the SR and DR, \citet{tse2025rewardaware} applied the principal eigenvector of the DR to similar applications, and found that the DR gave rise to better performance than the SR due to its reward-aware nature. \looseness=-1

The graph-drawing objective~\citep[GDO;][]{koren2003spectral} has been applied to learn the eigenvectors of the SR (or Laplacian) in RL~\citep{wulaplacian}. Extensions address problems with the GDO such as arbitrary rotations of the learned eigenvectors~\citep{wang2021towards, gomez2024proper}. However, similar methods do not exist for the DR. DROGO is the first to allow direct learning of the principal eigenvector of the DR. \looseness=-1


Apart from its application in computational RL, since the SR can be viewed as a middle ground between model-free and model-based decision making, it has received interest in behavioral and neuroscience studies~\citep{momennejad2017successor, stachenfeld2014design, stachenfeld2017hippocampus, gomez2025representations}. The DR was proposed originally in neuroscience to address, among others, the inflexibility of the SR in replanning tasks~\citep{piray2021linear}, and has been extended recently to incorporate compositionality~\citep{piray2025reconciling}.
Different from these works, we focus solely on computational RL. \looseness=-1



\section{Conclusion}
In this paper, we present DROGO, which optimizes a variant of the GDO with log-space parameterization and a point-wise constraint using natural gradient, and provide theoretical justification. Empirically, we demonstrate the effectiveness of DROGO in a number of environments with different state representations. We further demonstrate the utility of the learned eigenvectors by performing reward shaping. 

In this work, we run our experiments in grid world environments to be able to compare the eigenvectors learned using DROGO against the ground-truth computed in closed form by eigendecomposition.
In the future, it will be interesting to apply the learned eigenvectors for downstream applications, similar to how \citet{klissarov2023deep, chandrasekar2025towards} used GDO~\citep{wulaplacian, gomez2024proper} for exploration and generalization in RL.\looseness=-1

In our reward shaping experiments, we follow a two-stage approach, where we first learn the principal eigenvector before applying it for reward shaping. Such an approach serves our purpose in this paper: to demonstrate the utility of the approximated eigenvectors. However, as \citet{tse2025rewardaware} demonstrated the utility of the online approximation of the principal eigenvector of the DR in their option discovery experiments, future work can adapt DROGO to learn both the principal eigenvector and the policy simultaneously.

Given the recent prominence of Laplacian-/SR-based research in RL, DROGO unlocks promising new directions by scaling up a reward-aware generalization of the SR. \looseness=-1

\section*{Acknowledgements}
We thank Diego Gomez and Siddarth Chandrasekar for useful discussions. This research was supported in part by
the Natural Sciences and Engineering Research Council of Canada (NSERC), the Canada CIFAR
AI Chair Program, and Alberta Innovates. It was also enabled in part by computational resources
provided by the Digital Research Alliance of Canada. \looseness=-1

\section*{Impact Statement}
This paper presents work whose goal is to advance the field of Machine
Learning. There are many potential societal consequences of our work, none of
which we feel must be specifically highlighted here.

\nocite{langley00}

\bibliography{example_paper}
\bibliographystyle{icml2026}

\newpage
\appendix
\onecolumn
\section{An Alternative DR Formulation}
\label{appendix:formulation_diff}

\begin{figure}[ht]
  \vskip 0.2in
  \begin{center}
    \centerline{\includegraphics[width=\columnwidth]{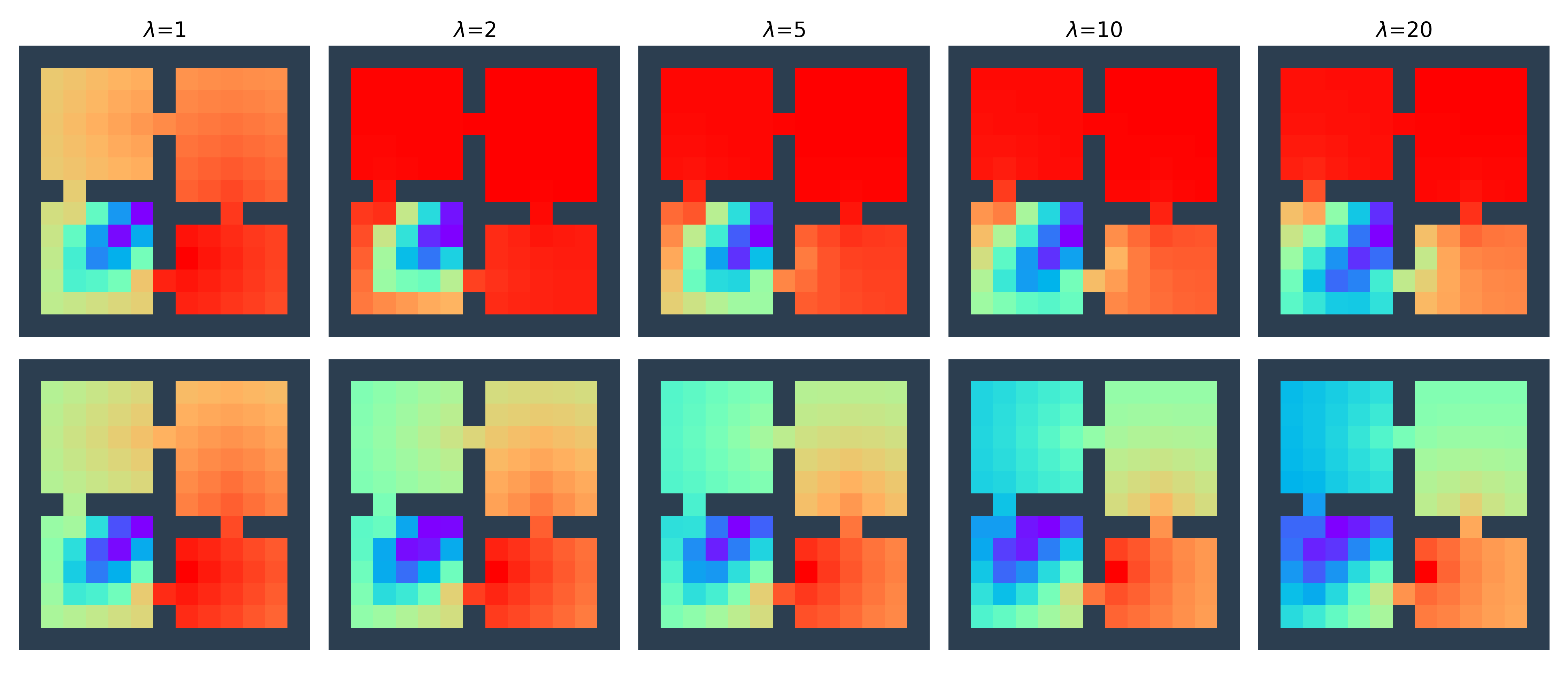}}
    \caption{
      The principal eigenvector of the DR for (top) not treating terminal states as absorbing states~\citep{tse2025rewardaware}, and (bottom) treating terminal states as absorbing states in the four rooms environment. See Appendix~\ref{appendix:environments} for a detailed explanation of the environments.
    }
    \label{fig:formulation_compare}
  \end{center}
\end{figure}

We formulate the DR slightly differently than prior work~\citep{tse2025rewardaware}. Specifically, we handle a terminal state differently. Let $N, T$ be the set of indices for non-terminal states and terminal states respectively. For example, $\mbf P_{NN}$ denotes the sub-matrix of $\mbf P$ that describes the transition probabilities from non-terminal states to non-terminal states. Recall that the DR is defined as
\begin{align}
    (\mbf R - \mbf P)\inv = \left ( \begin{bmatrix}
        \diag(\exp(-\mbf r_N  / \lambda)) & \mbf 0 \\ \mbf 0 & \diag(\exp(-\mbf r_T / \lambda))
    \end{bmatrix} - 
    \begin{bmatrix}
        \mbf P_{NN} & \mbf P_{NT} \\
        \mbf 0 & \mbf P_{TT}
    \end{bmatrix} \right) \inv .
\end{align}
\citet{tse2025rewardaware} allowed terminal states to emit any rewards at the last time step, but used 0 for most of their experiments, which is standard in RL. When the terminal state reward is 0, if we treat terminal states as absorbing states, i.e., transitioning to themselves with probability 1, then both $\diag(\exp(-\mbf r_T  / \lambda))$ and $\mbf P_{TT}$ are equal to $\mbf I$, and the DR is not well defined. To address this issue, \citet{tse2025rewardaware} considered terminal states as states that transition to a unique absorbing state with probability 1, and they removed this absorbing state from the DR matrix.

While this solution worked well for the experiments \citet{tse2025rewardaware} considered, where they used a small $\lambda=1.3$, we find that the DR formulated in this way is not robust to larger values of $\lambda$. Specifically, as shown on the top row of Fig.~\ref{fig:formulation_compare}, we find that as $\lambda$ increases, the principal eigenvector loses the property of having the largest value at the terminal state. A consequence is that, when we use such a principal eigenvector in applications such as reward shaping, the shaping reward encourages the agent to visit states other than the terminal states.

We find that we can handle terminal states differently such that we retain this property. Specifically, we treat terminal states as absorbing states, so $\mbf P_{TT} = \mbf I$. Then, to ensure that the DR is well defined, we require terminal states to have a reward of $- \delta$, where $\delta > 0$ is a tiny positive constant. We use $\delta = 0.001$ in our experiments. Using this formulation of the DR, as shown on the bottom row of Fig.~\ref{fig:formulation_compare}, the principal eigenvector is the largest at the terminal state even when $\lambda$ is large. 

Note that the theorems proved by \citet{tse2025rewardaware} can be easily adapted to this new formulation of the DR. Specifically, the symmetrized DR under the new formulation is still guaranteed to have a positive principal eigenvector. Finally, note that such formulations of the terminal states are merely part of the algorithm for computing the DR, and do not alter the underlying MDP. When evaluating the agent, for example, we still treat the environment as an episodic task that terminates with a 0 reward. \looseness=-1

\newpage
\section{Theory}
We provide derivations not included in the main text.

\subsection{Proof of Proposition~\ref{prop:gdo_loss}}
\label{appendix:proof_of_prop_gdo_loss}
\renewcommand{\thetheorem}{3.3}
\begin{proposition}
Given a transition $(s, a, r, s')$ and an extra state $s''$, where $s$ and $s''$ are sampled independently and uniformly at random, and $a$ is sampled under the default policy $\pi_d$. We can minimize Eq.~\ref{eq:dr_gdo_quad_summation} by minimizing the loss function
\begin{align}
    & \ell_{\text{GDO}}({\boldsymbol \theta}) =  \frac{1}{2} \big( u_{\boldsymbol \theta}(s) - u_{\boldsymbol \theta}(s') \big )^2 + \exp(-r / \lambda) u_{\boldsymbol \theta}(s)^2  + b' (u_{\boldsymbol \theta}(s)^2 - c') (u_{\boldsymbol \theta}(s'')^2 - c'),
\end{align}
    where $b'$ and $c'$ can be treated as hyperparameters.
\end{proposition}

\begin{proof}
    Let $n = |\mathcal{S}|$, and $\rho$ be the uniform state distribution. Eq.~\ref{eq:dr_gdo_quad_summation} can be expressed as
    \begin{align}
    & \placeholder \frac{1}{2}\sum_{s}\sum_{s'}\symP(s, s') \big(u(s) - u(s') \big)^2 + \sum_s \exp(-r(s) / \lambda) u(s)^2 + b \left ( \sum_{s}u(s)^2 - c \right )^2  \\
    &= \frac{n}{2} \sum_{s} \frac{1}{n} \sum_{s'}\symP(s, s') \big(u(s) - u(s') \big)^2  + n \sum_s \frac{1}{n} \exp(-r(s) / \lambda) u(s)^2 + b \left ( n \sum_{s} \frac{1}{n} u(s)^2 - c \right )^2 \\
    &= \frac{n}{2} \E_{s \sim \rho} \left[ \E_{s' \sim \tilde{\mbf P}(s, \cdot)}  \left[
     \big(u(s) - u(s') \big)^2
    \right]  \right] + n \E_{s \sim \rho} \left[\exp(-r(s) / \lambda) u(s)^2 \right] + bn^2 \left ( \E_{s \sim \rho}[ u(s)^2] - c / n \right )^2 \\
    & \propto \frac{1}{2} \E_{s \sim \rho} \left[ \E_{s' \sim \tilde{\mbf P}(s, \cdot)}  \left[
     \big(u(s) - u(s') \big)^2
    \right]  \right] + \E_{s \sim \rho} \left [\exp(-r(s) / \lambda) u(s)^2 \right] + bn \left ( \E_{s \sim \rho}[ u(s)^2 - c / n ]\right )^2 \\
    &= \frac{1}{2} \E_{s \sim \rho} \left[ \E_{s' \sim \tilde{\mbf P}(s, \cdot)}  \left[
     \big(u(s) - u(s') \big)^2
    \right]  \right] + \E_{s \sim \rho} \left [\exp(-r(s) / \lambda) u(s)^2 \right] + b' \left ( \E_{s \sim \rho}[ u(s)^2 - c' ]\right )^2. \label{eq:gdo_expectation}
\end{align}
Finally, $\ell_\text{GDO}$ is just a sample-based approximation of Eq.~\ref{eq:gdo_expectation}. Note that in the final sampling, we implicitly approximate $s' \sim \symP(s, \cdot)$ with $s' \sim \mbf P(s, \cdot)$, simply because sampling from $\symP$ requires sampling backward in time.
\end{proof}

\subsection{Proof of Proposition~\ref{prop:riemannian}}
\label{appendix:riemannian}

\renewcommand{\thetheorem}{3.5}
\begin{proposition}
    The distance in the $\mbf v$-space is characterized by the Riemannian metric tensor
    \begin{align}
        \mbf G(\mbf v) = \diag(\exp(2\mbf v)).
    \end{align}\looseness=-1
\end{proposition}

\begin{proof}

Suppose we have an optimization problem
\begin{align}
    \min_{\mbf u \in \mathcal{U}} L(\mbf u),
\end{align}
where $L$ is an objective function, and $\mathcal{U}$ is the parameter space for the optimization problem. 
We assume the Euclidean distance for $\mathcal{U}$. That is, the distance between two points $\mbf u_1, \mbf u_2 \in \mathcal{U}$ is computed as $d_\mathcal{U}(\mbf u_1, \mbf u_2) = \sqrt{(\mbf u_1 - \mbf u_2)^\top (\mbf u_1 - \mbf u_2)}$. Let $\delta \mbf u$ be a gradient step. The distance between the updated and the original parameters is $d_\mathcal{U}(\mbf u, \mbf u + \delta \mbf u) = \sqrt{\delta \mbf u^\top \delta \mbf u} = \|\delta \mbf u\|_2$, which is just the Euclidean norm of the gradient step. 


We can re-parameterize and optimize the logarithm of $\mbf u$ instead. We now have the following:
\begin{align}
    \min_{\mbf v \in \mathcal{V}} L(\exp(\mbf v)),
\end{align}
where $\exp$ here denotes element-wise exponentiation, and $\mathcal{V}$ is the parameter space for this new optimization problem. 

{Let $\delta \mbf u$ be the change induced in $\mbf u$ by $\delta \mbf v$, which is a small change in $\mbf v$. That is, $\delta \mbf u = \exp(\mbf v + \delta \mbf v) - \exp(\mbf v)$.}
Given 1) the relationship $\mbf u = \exp(\mbf v)$, and 2) the Euclidean distance for $\mathcal{U}$, the Riemannian metric tensor, $\mbf G(\mbf v)$, can be found by solving the following equation for small $\delta \mbf v$~\citep{amari1998natural2}:
\begin{align}
    d_\mathcal{U}^2(\mbf u, \mbf u+\delta \mbf u) = d_\mathcal{V}^2(\mbf v, \mbf v + \delta \mbf v) = \delta \mbf v^\top \mbf G(\mbf v) \delta \mbf v, \label{eq:riemannian_equation}
\end{align}
where $d_\mathcal{V}$ is the corresponding distance in the parameter space $\mathcal{V}$. See Fig.~\ref{fig:nat_grad} for a visual illustration.

\begin{figure}[h!]
  \vskip 0.2in
  \begin{center}
    \centerline{\includegraphics[width=0.25\columnwidth]{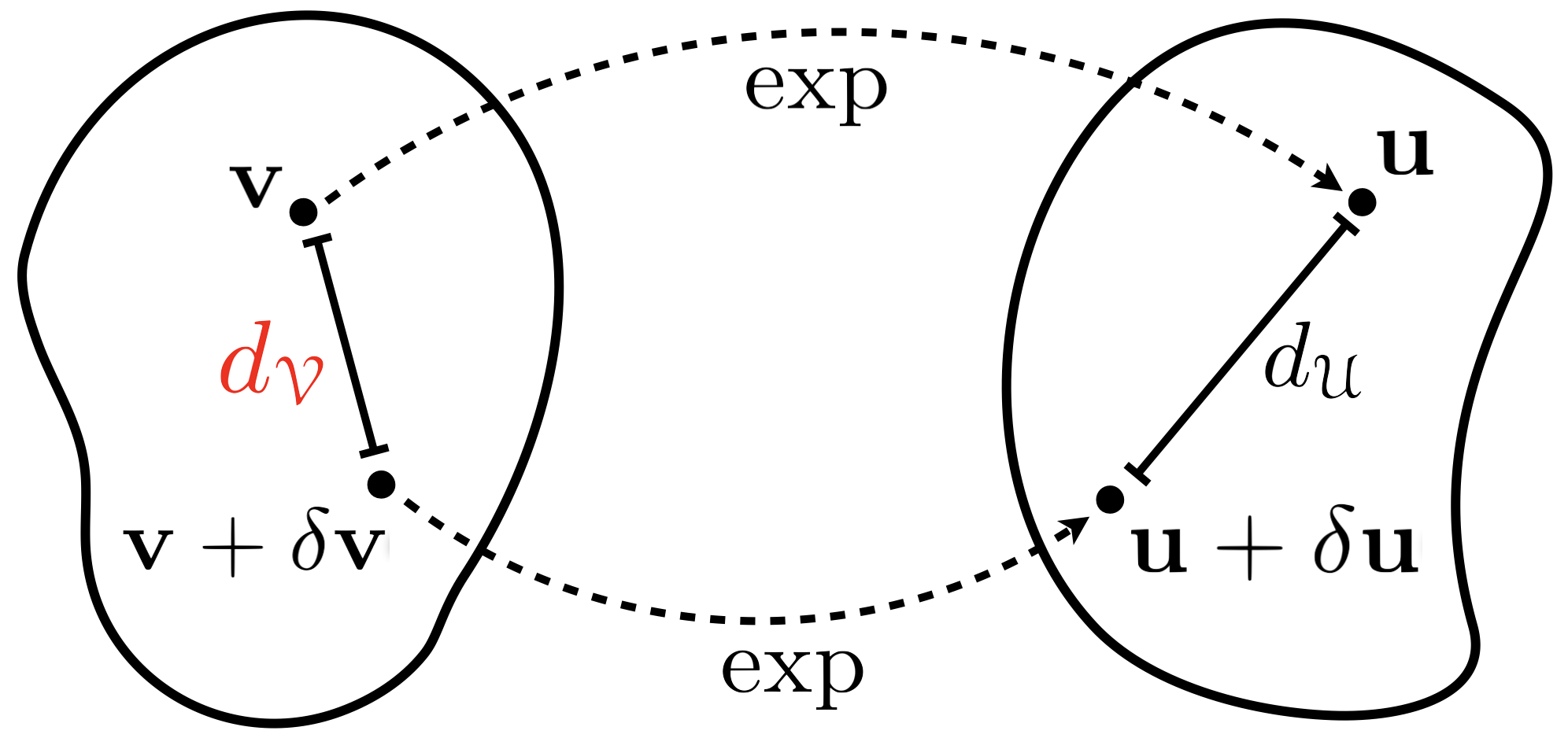}}
    \caption{
      $d_\mathcal{U}$ denotes the distance in the $\mbf u$-space, which is assumed to be Euclidean. We aim to find the induced distance in the $\mbf v$-space, denoted by $d_\mathcal{V}$, such that $d_\mathcal{V}$ between points in the $\mbf v$-space is measured by $d_\mathcal{U}$ of their images under the $\exp$ transformation.
    }
    \label{fig:nat_grad}
  \end{center}
\end{figure}

Starting from the left-hand side, we have
\begin{align}
    d^2_\mathcal{U}(\mbf u, \mbf u + \delta \mbf u)
    &=d^2_\mathcal{U} (\exp(\mbf v), \exp(\mbf v + \delta \mbf v)) \\
    &= \sum_s [\exp(v(s) + \delta v(s)) - \exp(v(s)) ]^2\\
    &= \sum_s [\exp(v(s)) (\exp(\delta v(s)) - 1)]^2 \\
    &= \sum_s [\exp(v(s)) (1 + \delta v(s) + O(\delta v(s)^2) - 1)]^2 \quad \quad \text{(Taylor expansion)} \\
    &\approx \sum_s [\exp(v(s)) \delta v(s)]^2  \quad \quad \text{($O(\delta v(s)^2)$ is ignored since $\delta \mbf v$ is small)} \\
    &= \sum_s \exp(2v(s)) \delta v(s)^2 \\
    &= \delta \mbf v^\top \operatorname{diag}(\exp(2 \mbf v)) \delta \mbf v.
\end{align}
Therefore, the Riemannian metric tensor is $\mbf G(\mbf v) = \diag( \exp(2 \mbf v))$.
\end{proof}


\subsection{Proof of Proposition~\ref{prop:nat_grad_second_term}}
\label{appendix:nat_grad_second_term_proof}
\renewcommand{\thetheorem}{3.7}
\begin{proposition}
    The natural gradient of $ L^{\log}_\text{GDO}(\mbf v_{\boldsymbol{\phi}})$ with respect to the output of the neural network parameterized by $\boldsymbol{\phi}$, $\mbf v_{\boldsymbol{\phi}}$, is proportional to
    \begin{align}
        \exp(-\mbf r / \lambda) -  \exp(-\mbf v_{\boldsymbol \phi}) \odot \symP \mbf u_{\boldsymbol \phi} + b'(\mbf u_{\boldsymbol \phi}^\top \mbf u_{\boldsymbol \phi} - c') \mbf 1.
    \end{align}
\end{proposition}
\begin{proof}
\begin{align}
    \frac{\partial L_\text{GDO}(\mbf u_{\boldsymbol{\phi}})}{\partial \mbf u_{\boldsymbol \phi}}
    &= \frac{\partial}{\partial \mbf u_{\boldsymbol \phi}} (\mbf u_{\boldsymbol \phi}^\top (\mbf I + \mbf R - \symP) \mbf u_{\boldsymbol \phi} + b(\mbf u^\top_{\boldsymbol \phi} \mbf u_{\boldsymbol \phi} - c)^2 )\\
    &\propto (\mbf I + \mbf R - \symP) \mbf u_{\boldsymbol \phi} + 2b (\mbf u_{\boldsymbol \phi}^\top \mbf u_{\boldsymbol \phi} - c)\mbf u_{\boldsymbol \phi} \\
    &= (\mbf R - \symP) \mbf u_{\boldsymbol \phi} + [b'(\mbf u^\top_{\boldsymbol \phi} \mbf u_{\boldsymbol \phi} - c) + b'/b'] \mbf u_{\boldsymbol \phi} \\
    &= (\mbf R - \symP)\mbf u_{\boldsymbol \phi} + b'(\mbf u_{\boldsymbol \phi}^\top \mbf u_{\boldsymbol \phi} - c') \mbf u_{\boldsymbol \phi}, \label{eq:nat_grad_second_term}
\end{align}
where we absorb constants into $b'$ and $c'$. Then, the Hadamard product of $\exp(-\mbf v_{\boldsymbol{\phi}})$ with Eq.~\ref{eq:nat_grad_second_term} is
\begin{align}
    &\placeholder \exp(-\mbf v_{\boldsymbol \phi}) \odot [(\mbf R - \symP)\mbf u_{\boldsymbol \phi} + b'(\mbf u_{\boldsymbol \phi}^\top \mbf u_{\boldsymbol \phi} - c') \mbf u_{\boldsymbol \phi} ] 
    \\&= \exp(-\mbf v_{\boldsymbol \phi}) \odot  (\mbf R - \symP)\mbf u_{\boldsymbol \phi} + b'(\mbf u_{\boldsymbol \phi}^\top \mbf u_{\boldsymbol \phi} - c') \mbf 1 \\
    &= \exp(-\mbf v_{\boldsymbol \phi}) \odot \mbf R \mbf u_{\boldsymbol \phi} -  \exp(-\mbf v_{\boldsymbol \phi}) \odot \symP \mbf u_{\boldsymbol \phi}  + b'(\mbf u_{\boldsymbol \phi}^\top \mbf u_{\boldsymbol \phi} - c') \mbf 1 \\
    &=  \exp(-\mbf v_{\boldsymbol \phi}) \odot \exp(-\mbf r / \lambda) \odot \mbf u_{\boldsymbol \phi}   -  \exp(-\mbf v_{\boldsymbol \phi}) \odot \symP \mbf u_{\boldsymbol \phi} + b'(\mbf u_{\boldsymbol \phi}^\top \mbf u_{\boldsymbol \phi} - c') \mbf 1 \\
    &=  \exp(-\mbf r / \lambda) -  \exp(-\mbf v_{\boldsymbol \phi}) \odot \symP \mbf u_{\boldsymbol \phi} + b'(\mbf u_{\boldsymbol \phi}^\top \mbf u_{\boldsymbol \phi} - c') \mbf 1.
\end{align}
\end{proof}

\subsection{Proof of Proposition~\ref{prop:s_th_entry_estimate}}
\label{appendix:s_th_entry_proof}

\renewcommand{\thetheorem}{3.8}
\begin{proposition}
    The $s$-th entry of Eq.~\ref{eq:dr_gdo_quad_log_grad} is estimated by \looseness=-1
    \begin{align}
    \exp(-r/\lambda) - \exp(v_{\boldsymbol \phi}(s') - v_{\boldsymbol \phi}(s)) + b''(\exp(2v_{\boldsymbol \phi}(s'')) - c''), \label{eq:62}
\end{align}
where $a$ in $(s, a, r, s')$ is sampled from the default policy, $\pi_d$, and $s''$ is a state sampled uniformly at random. \looseness=-1
\end{proposition}

\begin{proof}
    Let $n=|\mathcal{S}|$, and $\rho$ be the uniform state distribution. The $s$-th entry of Eq.~\ref{eq:dr_gdo_quad_log_grad} is
    \begin{align}
        & \placeholder [\exp(-\mbf r / \lambda) -  \exp(-\mbf v_{\boldsymbol \phi}) \odot \symP \mbf u_{\boldsymbol \phi} + b'(\mbf u_{\boldsymbol \phi}^\top \mbf u_{\boldsymbol \phi} - c') \mbf 1](s) \\
        &= \exp(-r(s) / \lambda) - \exp(-v_{\boldsymbol{\phi}}(s)) [\symP \mbf u_{\boldsymbol \phi}](s) + b'(\mbf u_{\boldsymbol \phi}^\top \mbf u_{\boldsymbol \phi} - c')  \\ 
        &= \exp(-r(s) / \lambda) - \exp(-v_{\boldsymbol{\phi}}(s)) \E_{s' \sim \symP(s, \cdot)}[u_{\boldsymbol{\phi}}(s')] + b' \left(\sum_{s''} u_{\boldsymbol{\phi}}^2(s'') - c' \right) \\
        &= \exp(-r(s) / \lambda) - \exp(-v_{\boldsymbol{\phi}}(s)) \E_{s' \sim \symP(s, \cdot)}[\exp(v_{\boldsymbol{\phi}}(s'))] + b' \left(n\sum_{s''} \frac{1}{n} u_{\boldsymbol{\phi}}^2(s'') - c' \right) \\
        &= \exp(-r(s) / \lambda) - \E_{s' \sim \symP(s, \cdot)}[\exp(v_{\boldsymbol{\phi}}(s') -v_{\boldsymbol{\phi}}(s))] + b'n \left(\E_{s'' \sim \rho}[ u_{\boldsymbol{\phi}}^2(s'')] - c' /n\right)  \\
        &= \exp(-r(s) / \lambda) - \E_{s' \sim \symP(s, \cdot)}[\exp(v_{\boldsymbol{\phi}}(s') -v_{\boldsymbol{\phi}}(s))] + b'' \left(\E_{s'' \sim \rho}[ \exp(2v_{\boldsymbol{\phi}}(s''))] - c''\right).
    \end{align}
    Eq.~\ref{eq:62} is a sample-based approximation of the above equation.
\end{proof}

\subsection{Proof of Proposition~\ref{prop:surrogate_loss}}
\label{appendix:surrogate_loss_proof}
\renewcommand{\thetheorem}{3.9}
\begin{proposition}
Given a transition $(s, a, r, s')$ and an extra state $s''$, where $s$ and $s''$ are sampled independently and uniformly at random, and $a$ is sampled under the default policy $\pi_d$. We can minimize $ L^{\log}_\text{GDO}(\mbf v_{\boldsymbol{\phi}})$ with natural gradient descent by minimizing 
\begin{align}
    \ell^{\text{NG}}_{\text{GDO}}({\boldsymbol \phi}) = \llbracket \exp(-r(s) / \lambda) - \exp(v_{\boldsymbol \phi}(s') - v_{\boldsymbol \phi}(s)) 
     + b''(\exp(2v_{\boldsymbol \phi}(s'')) - c'') \rrbracket v_{\boldsymbol \phi}(s), 
\end{align}
where $\llbracket \cdot \rrbracket$ denotes stop-gradient, and $\lambda, b'', c''$ are hyperparameters.
\end{proposition}
\begin{proof}
    Let $n = |\mathcal{S}|$, and $\rho$ be the uniform state distribution. Recall that the gradient of $ L^{\log}_\text{GDO}(\mbf v_{\boldsymbol{\phi}})$ with respect to the $i$-th entry of the parameters, $\boldsymbol{\phi}$, is (see Eq.~\ref{eq:dr_quad_nat_grad})
    \begin{align}
     & \placeholder   \frac{\partial \mbf v_{\boldsymbol \phi}}{\partial \phi_i} \cdot \left(
    \exp(- \mbf v_{\boldsymbol \phi}) \odot  \frac{\partial L_{\text{GDO}}(\mbf u_{\boldsymbol{\phi}})}{\partial \mbf u_{\boldsymbol \phi}} \right ) \\
    &= \sum_s  \left [ \frac{\partial \mbf v_{\boldsymbol \phi}}{\partial \phi_i} \right] (s)  \left[
    \exp(- \mbf v_{\boldsymbol \phi}) \odot  \frac{\partial L_{\text{GDO}}(\mbf u_{\boldsymbol{\phi}})}{\partial \mbf u_{\boldsymbol \phi}} \right ](s) \\
    &= n \sum_s \frac{1}{n} \frac{\partial v_{\boldsymbol \phi}(s) }{\partial \phi_i} \left[
    \exp(- \mbf v_{\boldsymbol \phi}) \odot  \frac{\partial L_{\text{GDO}}(\mbf u_{\boldsymbol{\phi}})}{\partial \mbf u_{\boldsymbol \phi}} \right ](s) \\
    &\propto  \E_{s \sim \rho} \left [
    \frac{\partial v_{\boldsymbol \phi}(s) }{\partial \phi_i}  \left[
    \exp(- \mbf v_{\boldsymbol \phi}) \odot  \frac{\partial L_{\text{GDO}}(\mbf u_{\boldsymbol{\phi}})}{\partial \mbf u_{\boldsymbol \phi}} \right ](s)
    \right ]
    \end{align}
We can estimate this expectation by sampling a state, $s$, uniformly at random. Applying Propositions~\ref{prop:nat_grad_second_term} and~\ref{prop:s_th_entry_estimate} gives
\begin{align}
    &\placeholder \frac{\partial v_{\boldsymbol \phi}(s) }{\partial \phi_i}  \left[
    \exp(- \mbf v_{\boldsymbol \phi}) \odot  \frac{\partial L_{\text{GDO}}(\mbf u_{\boldsymbol{\phi}})}{\partial \mbf u_{\boldsymbol \phi}} \right ](s) \\
    &= \frac{\partial v_{\boldsymbol \phi}(s) }{\partial \phi_i} \big (\exp(-r/\lambda) - \exp(v_{\boldsymbol \phi}(s') - v_{\boldsymbol \phi}(s)) + b''(\exp(2v_{\boldsymbol \phi}(s'')) - c'' ) \big ), \label{eq:72}
\end{align}
where $(s, a, r, s')$ is the transition generated by sampling $a$ from the default policy, $\pi_d$, and $s''$ is another state sampled uniformly at random. Eq.~\ref{eq:72} is a sampled-based approximation of the gradient of $ L^{\log}_\text{GDO}(\mbf v_{\boldsymbol{\phi}})$ with respect to $\phi_i$. Finally, it is trivial to show that the gradient of $\ell^{\text{NG}}_\text{GDO}(\boldsymbol{\phi})$ recovers Eq.~\ref{eq:72}.
\end{proof}

\subsection{Proof of Theorem~\ref{thm:eigenspace}}
\label{appendix:eigenspace}

\renewcommand{\thetheorem}{3.11}
\begin{theorem}
  Let $r(s_T) = -\delta\ \forall s_T \in \mathcal{S}_T$, where $\delta > 0$. When $\delta\to0^+$, the columns of the DR ($(\mbf R - \mbf P)\inv$) corresponding to all terminal states form a basis for the eigenspace of the principal eigenvalue of the DR.
\end{theorem}

\begin{proof}

    Recall that the (non-symmetrized) DR is $(\mbf R - \mbf P)\inv = (\diag(\exp(-\mbf r / \lambda) - \mbf P)\inv$. We assume without loss of generality that $\lambda=1$.
    Let $n = |\mathcal{S}|$ and $m = |\mathcal{S}_T|$.
    Suppose we order the states such that the $m$ terminal states are ordered last, i.e., they correspond to the last columns of the DR. 
    Then, $\mbf P = \begin{bmatrix}
        \mbf A & \mbf B\\ \mbf 0 &  \mbf I
    \end{bmatrix},$ for some matrices $\mbf A$ and $\mbf B$,
    and $\mbf R - \mbf P = \begin{bmatrix}
        \mbf A' & -\mbf B \\ \mbf 0  & (\exp(\delta) - 1) \mbf I
    \end{bmatrix}$ for some matrix $\mbf A'$.  
    When $\delta \to 0^+$, $\exp(\delta) - 1 \approx \delta$, so we write $\mbf R - \mbf P$ as $ \begin{bmatrix}
        \mbf A' & -\mbf B \\ \mbf 0  & \delta \mbf I
    \end{bmatrix}$.

    It is easy to see that $\delta$ is an eigenvalue of $\mbf R - \mbf P$ with algebraic multiplicity $m$. Note that $\mbf R - \mbf P$ might have complex eigenvalues, and we order eigenvalues by their magnitudes.
    Using the Gershgorin circle theorem, we can show that all eigenvalues of $\mbf R - \mbf P$ have non-zero magnitude. It is obvious that for the rows corresponding to terminal states, the Gershgorin discs are centered at $\delta$ with radius 0. For a row $i$ corresponding to a non-terminal state, the Gershgorin disc is centered at $\exp(-r(s_i)) - P(i, i)$. Since we assume $r(s) < 0\ \forall s \in \mathcal{S}$, $\exp(-r(s_i)) - P(i,i) > 1 - P(i, i) = \sum_{j\neq i}P(i, j)$, where $\sum_{j\neq i} P(i, j)$ is the radius of the disc. Thus, any discs corresponding to non-terminal states do not contain $0$. Then, all eigenvalues of $\mbf R - \mbf P$ must have magnitudes greater than $0$.
    Since $\delta$ can be arbitrarily small, $\delta$ is the eigenvalue of $\mbf R - \mbf P$ with the smallest magnitude, and thus the principal eigenvalue of the DR, $(\mbf R - \mbf P)\inv$. 

    We now study the principal eigenvectors of the DR, which are equal to the eigenvectors of $\mbf R - \mbf P$ corresponding to $\delta$. An eigenvector corresponding to $\delta$ satisfies
    \begin{align}
        \begin{bmatrix}
            \mbf A' & -\mbf B \\ \mbf 0 &  \delta  \mbf I
        \end{bmatrix} \begin{bmatrix}
            \mbf x \\ \mbf y
        \end{bmatrix} = 
        \delta \begin{bmatrix}
            \mbf x \\ \mbf y
        \end{bmatrix},
    \end{align}
    which gives the equation 
    \begin{align}
        \mbf A' \mbf x - \mbf B \mbf y &= \delta \mbf x \\
        \mbf x &= (\mbf A' - \delta \mbf I)\inv \mbf B \mbf y
    \end{align}
    An eigenvector corresponding to $\delta$ then has the general form
    \begin{align}
        \begin{bmatrix}
            (\mbf A' - \delta \mbf I)\inv \mbf B \mbf y \\\mbf y 
        \end{bmatrix}, \label{eq:dr_eigvec_form}
    \end{align}
    where $\mbf y$ can be any vector in $\mathbb{R}^{m}$. Therefore, the geometric multiplicity of $\delta$ is also $m$.

    We now derive the form of the columns of the DR corresponding to terminal states. By block-wise matrix inversion, the DR is \looseness=-1
    \begin{align}
        \begin{bmatrix}
            \mbf A' & -\mbf B \\ \mbf 0  & \delta \mbf I
        \end{bmatrix} \inv = 
        \begin{bmatrix}
            (\mbf A')\inv & \frac{1}{\delta} (\mbf A')\inv \mbf B \\ \mbf 0 & \frac{1}{\delta} \mbf I
        \end{bmatrix}.
    \end{align}
    Then, the columns corresponding to the terminal states have the form
    \begin{align}
        \begin{bmatrix}
            \frac{1}{\delta} (\mbf A')\inv \mbf B \mbf b_j \\ 
            \frac{1}{\delta} \mbf b_j, 
        \end{bmatrix} \label{eq:dr_column_form}
    \end{align}
    where $\mbf b_i \in \mathbb{R}^m$ is the standard basis vector with a $1$ at entry $i$, and $j \in \{1, 2, \dots, m\}$. Comparing Eq.~\ref{eq:dr_eigvec_form} and~\ref{eq:dr_column_form}, we can see that as $\delta \to 0^+$, these linearly independent columns of the DR span the same space as the eigenvectors corresponding to the eigenvalue $\delta$. Thus, they form a basis for the eigenspace of the principal eigenvalue of the DR.
\end{proof}





\newpage
\section{Experiment Details}
\label{appendix:exp_detail}
We present more experiment details.

\subsection{Environments}
\label{appendix:environments}

We use the set of environments shown in Fig.~\ref{fig:env}.

\begin{figure}[ht]
  \vskip 0.2in
  \begin{center}
    \centerline{\includegraphics[width=0.8\textwidth]{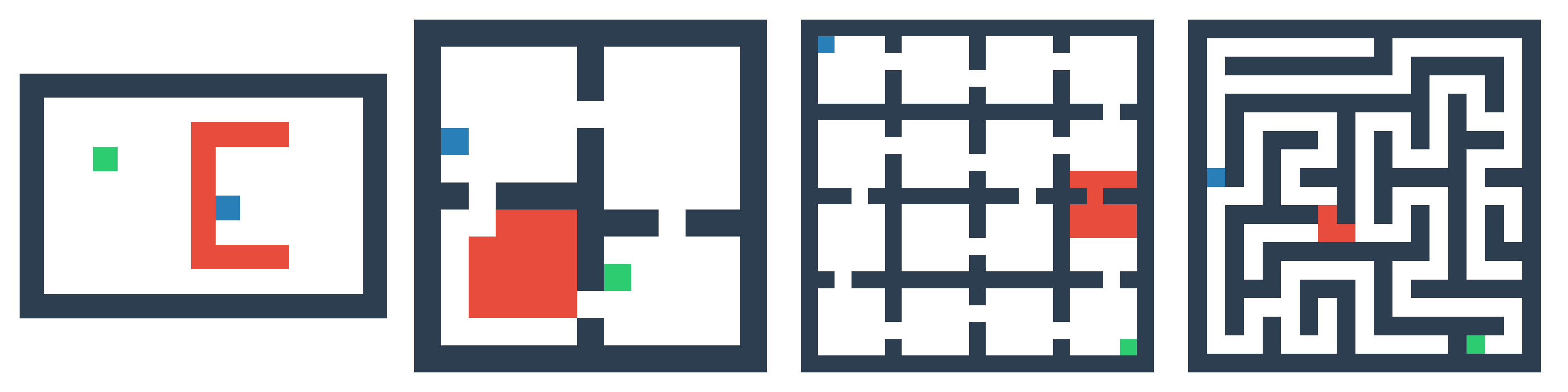}}
    \caption{
      The set of episodic environments used by \citet{tse2025rewardaware}. From left to right: 1) {grid task}, 2) four rooms, 3) grid room, and 4) grid maze. The start state is in blue. The agent receives $-1$ reward at every time step unless it steps on red tiles ($-20$ reward) or reaches the goal in green ($0$ reward).
    }
    \label{fig:env}
  \end{center}
\end{figure}

\subsection{Details on DROGO Experiments}
\label{appendix:drogo_details}
We provide more details on the experiments using DROGO to learn the principal eigenvector of the DR.

\subsubsection{Neural Network Architecture}
\label{appendix:neural_net}
We use the same neural network architecture for one-hot and coordinates representations of states. The network consists of 4 fully connected layers with 128 hidden units followed by a final output layer. For pixel inputs, the network consists of 3 convolutional layers (16, 32, and 64 filters respectively) with $3 \times 3$ kernels, stride 1, and padding of 1. The first two convolutional layers are each followed by $2 \times 2$ max-pooling with stride 2. The output is then flattened and fed to a fully connected layer with 128 hidden units, followed by a final output layer.
All fully connected and convolutional layers, except for the output layer, are followed by the rectified linear unit (ReLU) activation~\citep{nair2010rectified}.

\subsubsection{Hyperparameters}
\label{appendix:hyperparam}
We tune the hyperparameters manually. We use the same set of hyperparameters for all environments and all state representations. We use $\lambda=20$, which prevents the $\exp(-r(s) / \lambda)$ term in the DROGO loss (Eq.~\ref{eq:drogo_loss}) from being too large and causing instabilities. We use the RMSprop optimizer~\cite{hinton2012rmsprop}, a step size of 1e$-5$, and a batch size of $2000$. We apply gradient norm clipping of 0.5. For coordinates and pixel inputs, we normalize the inputs to have the range $[-0.5, 0.5]$.

\subsubsection{Visualization of the Learned and Ground-truth Principal Eigenvectors}
\label{appendix:visual_comparison}

We visually compare the learned eigenvectors and the ground-truth principal eigenvectors of the DR in Fig.~\ref{fig:eigvec_vis}.

\begin{figure}[t!]
  \vskip 0.2in
  \begin{center}
    \centerline{\includegraphics[width=\textwidth]{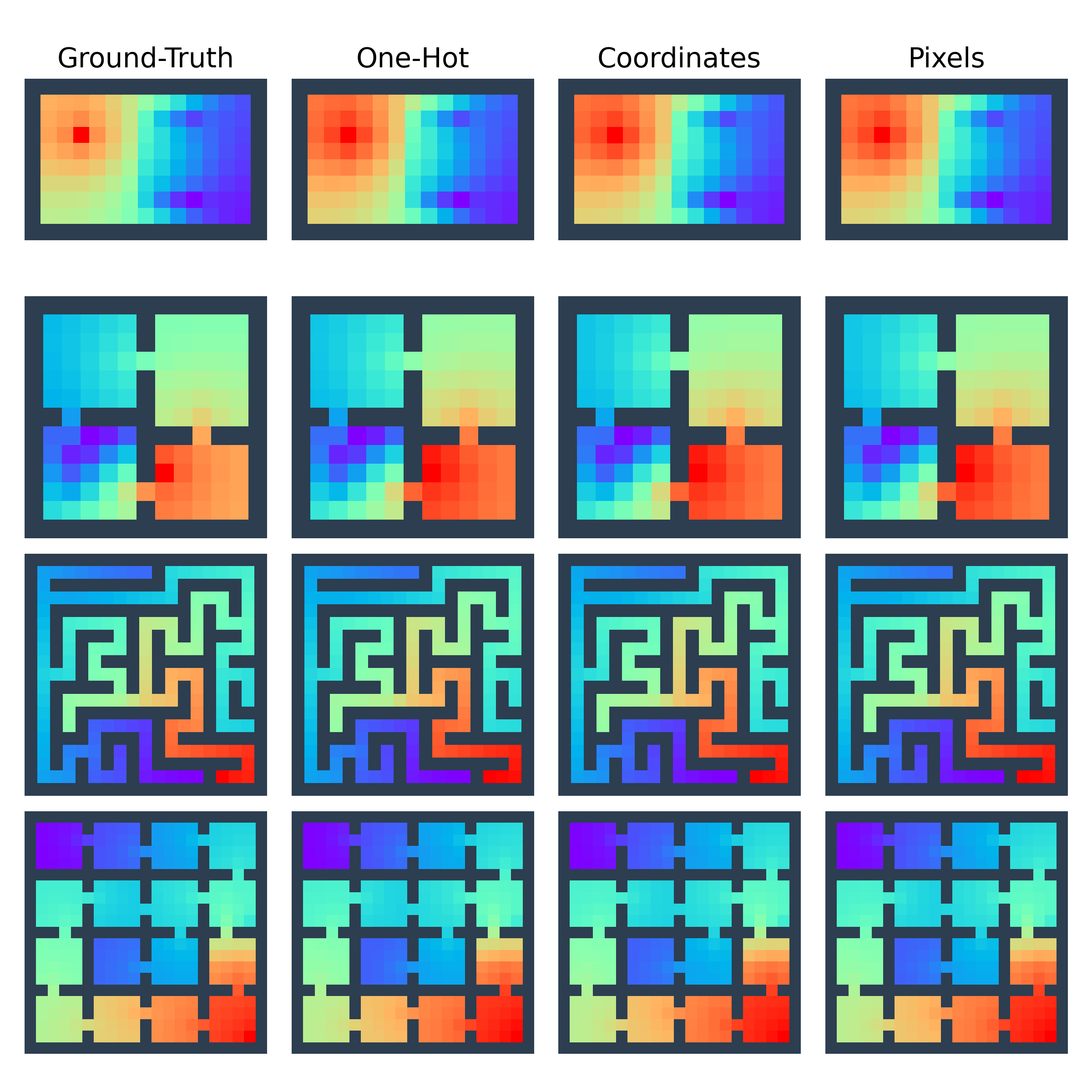}}
    \caption{
      Visualization of the vectors learned using DROGO, averaged over 10 seeds, for different state representations. Columns from left to right: 1) the ground-truth principal eigenvectors of the DR, 2) the eigenvectors learned for one-hot inputs, 3) the eigenvectors learned for coordinates inputs, and 4) the eigenvectors learned for pixels inputs. Note that the eigenvectors here appear visually different from those presented by \citet{tse2025rewardaware} since they used $\lambda=1.3$, while we use $\lambda=20$.
    }
    \label{fig:eigvec_vis}
  \end{center}
\end{figure}

\newpage

\quad
\newpage

\subsection{Ablation Study}
\label{appendix:ablation}

\begin{table*}[t]
  \caption{The characterization of different loss functions.}
  \label{table:loss_fn}
  \begin{center}
    \begin{small}
      \begin{sc}
        \renewcommand{\arraystretch}{1.2}
        \begin{tabular}{lccc}
          \toprule
          Loss Function  & \makecell{Log-Space \\ Parameterization}  & Natural Gradient & \makecell{Quadratic \\ Norm Penalty}    \\
          \midrule
          $\ell_\text{GDO}$ & \cross & \cross & \tick \\
          $\ell^{\log}_{\text{GDO}}$ & \tick & \cross & \tick \\
          $\ell^{\text{NG}}_{\text{GDO}}$ & \tick & \tick & \tick \\
          $\ell_{\text{DROGO}}$ & \tick & \tick & \cross \\
          \bottomrule
        \end{tabular}
      \end{sc}
    \end{small}
  \end{center}
  \vskip -0.1in
\end{table*}

We empirically compare $\ell_\text{DROGO}$ (Eq.~\ref{eq:drogo_loss}), ${\ell}^{\text{NG}}_{\text{GDO}}$ (Eq.~\ref{eq:dr_quad_surr_loss}), ${\ell}^{\log}_\text{GDO}$ (Eq.~\ref{eq:dr_gdo_quad_log_loss}), and $\ell_\text{GDO}$ (Eq.~\ref{eq:dr_gdo_quad_loss}). Recall that only $\ell_\text{GDO}$ does not perform log-space parameterization. Of the remaining three, ${\ell}^{\log}_\text{GDO}$ does not adopt natural gradient. Finally, $\ell_{\text{DROGO}}$ uses the point-wise constraint, while ${\ell}^{\text{NG}}_\text{GDO}$ uses the quadratic norm penalty. See Table~\ref{table:loss_fn} for a summary. 

To compare these loss functions, we compute the cosine similarity between the vectors learned using the loss functions and the logarithm of the ground-truth principal eigenvector of the DR. Note that for $\ell_\text{GDO}$, there is no guarantee that we can safely apply the logarithm to the learned vector, so we apply the logarithm to the absolute values of the learned vector.

We first compare these loss functions in the tabular case, where we represent the learned vector using an array of length $|\mathcal{S}|$. Let $b$ be the hyperparameter determining the strength of the quadratic penalty term, and $\alpha$ be the step size. For $\ell_\text{GDO}$, ${\ell}^{\log}_\text{GDO}$, and ${\ell}^{\text{NG}}_\text{GDO}$, we perform a grid search over $\alpha \in [10^{-5}, 3\cdot 10^{-5}, 10^{-4}, 3 \cdot 10^{-4}, 10^{-3}]$ and $b \in [0.1, 0.5, 1, 2]$. For $\ell_\text{DROGO}$, which does not require $b$, we simply perform a grid search over the same values of $\alpha$. Note that we perform the hyperparameter search in the largest environment, grid room, and use the best hyperparameters found for all environments. 

Fig.~\ref{fig:tabular_ablation} shows the cosine similarity between the learned and ground-truth principal eigenvectors of the DR for the different loss functions. Apart from ${\ell}^{\log}_\text{GDO}$, all loss functions learn vectors that converge to the ground-truth principal eigenvector. Note that $\ell^{\text{NG}}_{\text{GDO}}$ appears to have poor performance in grid room since one out of 10 seeds fails to converge. $\ell_\text{GDO}$, which simply applies the GDO without log-space parameterization, performs well in the tabular setting. However, as we will see later, this loss function does not work well with neural networks. For the remaining loss functions, ${\ell}^{\log}_{\text{GDO}}$ fails to converge to the ground-truth eigenvector, while ${\ell}^{\text{NG}}_\text{GDO}$ and $\ell_\text{DROGO}$ does, confirming the importance of the natural gradient. 

\begin{figure}[ht]
  \vskip 0.2in
  \begin{center}
    \centerline{\includegraphics[width=\textwidth]{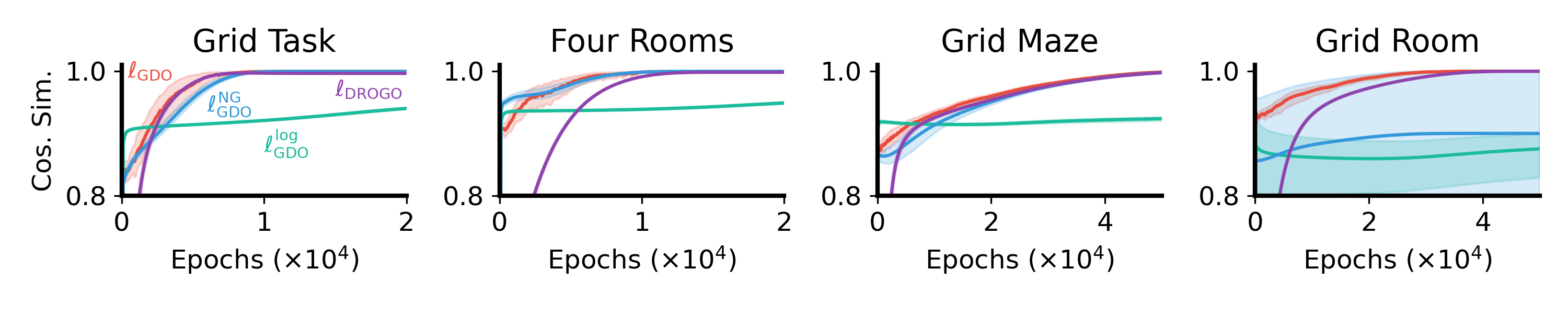}}
    \caption{
      The cosine similarity, averaged over 10 seeds, between the learned and ground-truth principal eigenvectors for different loss functions in the tabular setting. The shaded area indicates 95\% bootstrap confidence interval.
    }
    \label{fig:tabular_ablation}
  \end{center}
\end{figure}

We now compare the loss functions using neural networks taking the coordinates of states as inputs. For $\ell_\text{DROGO}$, we tune the hyperparameters manually, as detailed in Appendix~\ref{appendix:drogo_details}. Because we tuned more hyperparameters, we considered fewer values per hyperparameter, effectively considering fewer hyperparameter settings than the grid search described above. For the remaining loss functions, we use the same hyperparameters as $\ell_\text{DROGO}$ except for the step size, $\alpha$, and the quadratic penalty coefficient, $b$. We identify the values for $\alpha$ and $b$ using the same search procedure as in the tabular case. 

Figure~\ref{fig:ablation} shows the cosine similarity between the learned and ground-truth principal eigenvectors of the DR. We can see that only $\ell_\text{DROGO}$ can learn the principal eigenvector. While in the tabular setting, we observe that vectors learned using $\ell_\text{GDO}$ converge to the ground-truth principal eigenvector, the approximation errors with neural networks and the small magnitudes of the ground-truth eigenvector prevent it from learning the principal eigenvector well in the function approximation case. We empirically observe that the vectors learned using $\ell_\text{GDO}$ have negative values. ${\ell}^{\log}_\text{GDO}$, on the other hand, fails for the same reason of lacking natural gradient as in the tabular setting. Finally, using the quadratic norm penalty (${\ell}^{\text{NG}}_\text{GDO}$) fails due to the difficult of balancing $b$. When $b$ is large, the penalty interferes with the GDO. When $b$ is small, it is not sufficient to constrain the norm of the learned vector. 

\begin{figure}[ht]
  \vskip 0.2in
  \begin{center}
    \centerline{\includegraphics[width=\textwidth]{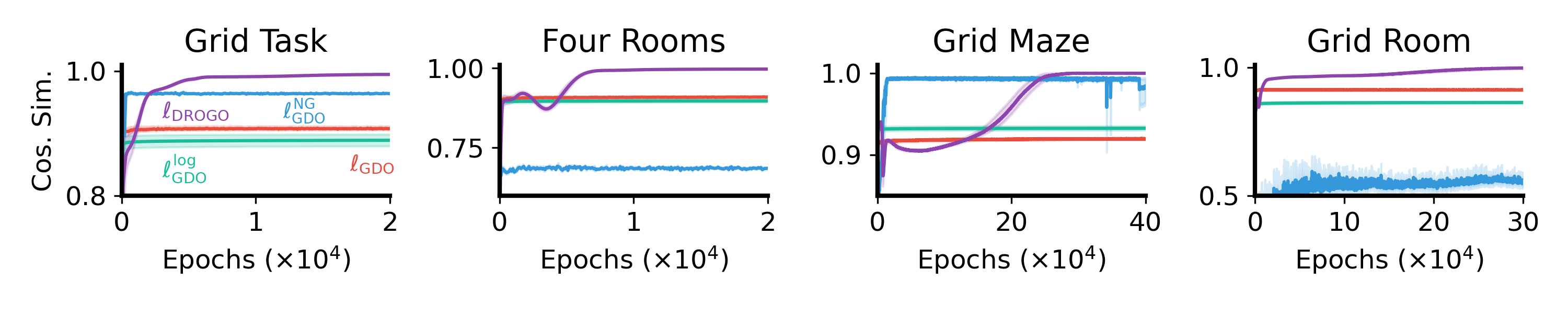}}
    \caption{
      The cosine similarity, averaged over 10 seeds, between the learned and ground-truth principal eigenvectors for different loss functions when using coordinates as inputs to the neural network. The shaded area indicates 95\% bootstrap confidence interval.
    }
    \label{fig:ablation}
  \end{center}
\end{figure}

\subsection{Learning Curves for Reward Shaping Experiments}
\label{appendix:reward_shaping}
We present the learning curves for the reward shaping experiments in Fig.~\ref{fig:reward_shaping}. Since the learning curves for different state representations of the DR highly overlap, we here present only the learning curve for the pixels representation of states. 

\begin{figure}[ht]
  \vskip 0.2in
  \begin{center}
    \centerline{\includegraphics[width=\textwidth]{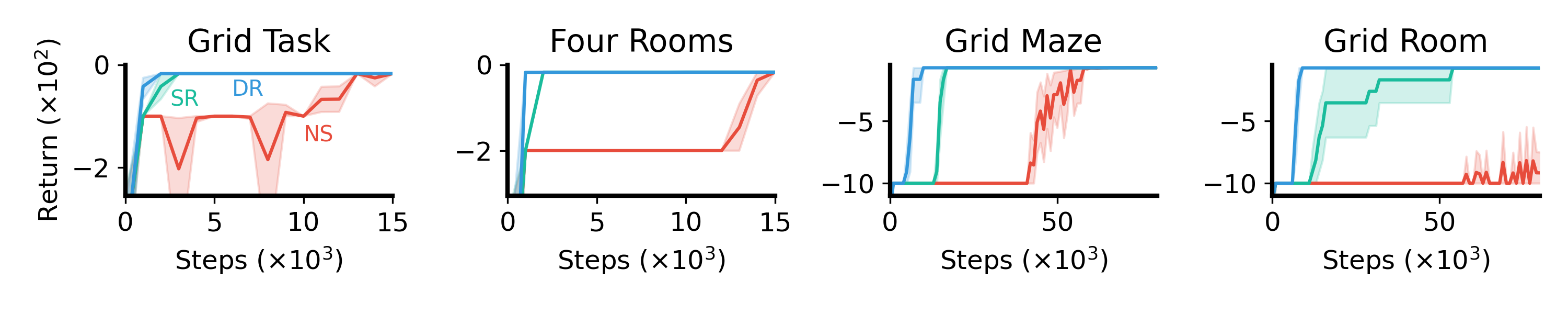}}
    \caption{
      The total return, averaged over 10 seeds, obtained when using the DR for reward shaping, the SR for reward shaping, and not performing reward shaping (NS). The shaded area indicates 95\% bootstrap confidence interval.
    }
    \label{fig:reward_shaping}
  \end{center}
\end{figure}

\end{document}